\documentclass{article}

\usepackage{microtype}

\usepackage{hyperref}

\usepackage{graphicx}
\usepackage{subcaption}
\usepackage{booktabs}
\usepackage{float}
\usepackage{enumerate}

\usepackage{amsmath}
\usepackage{amssymb}
\usepackage{mathtools}
\usepackage{amsthm}

\usepackage[capitalize,noabbrev]{cleveref}

\usepackage{tikz}
\usetikzlibrary{positioning,arrows.meta,fit,backgrounds}

\usepackage[disable,textsize=tiny]{todonotes}

\usepackage{lmodern}
\usepackage[T1]{fontenc}
\usepackage{url}
\usepackage{xurl}
\newcommand{\id}{\textnormal{id}}

\newcommand{\Q}{\mathbb{Q}}
\newcommand{\Z}{\mathbb{Z}}
\newcommand{\N}{\mathbb{N}}

\newcommand{\cL}{\mathcal{L}}
\newcommand{\poly}{\textnormal{poly}}
\newcommand{\cX}{\mathcal{X}}
\newcommand{\cY}{\mathcal{Y}}
\newcommand{\cH}{\mathcal{H}}
\newcommand{\cD}{\mathcal{D}}
\newcommand{\cA}{\mathcal{A}}

\theoremstyle{plain}
\newtheorem{theorem}{Theorem}[section]

\newtheorem{lemma}[theorem]{Lemma}
\newtheorem{conjecture}[theorem]{Conjecture}
\newtheorem{corollary}[theorem]{Corollary}
\newtheorem{claim}[theorem]{Claim}
\newtheorem{problem}[theorem]{Problem}

\theoremstyle{definition}
\newtheorem{definition}[theorem]{Definition}

\theoremstyle{remark}
\newtheorem{remark}[theorem]{Remark}

\title{Why ReLU? A Bit-Model Dichotomy for Deep Network Training}

\author{
	Ilan Doron-Arad\thanks{Supported by grant NSF DMS-2031883 and Vannevar Bush Faculty Fellowship ONR-N00014-20-1-2826 (PI Mossel).}\\
	MIT, Cambridge, MA, USA\\
	\texttt{ilanda@mit.edu}
	\and
	Elchanan Mossel\thanks{Partially supported by NSF DMS-2031883, the Vannevar Bush Faculty Fellowship ONR-N00014-20-1-2826, and a Simons Investigator Award.}\\
	MIT, Cambridge, MA, USA\\
	\texttt{elmos@mit.edu}
}

\date{}

\begin{document}
\maketitle

\begin{abstract}
Theoretical analyses of Empirical Risk Minimization (ERM) are standardly framed within the Real-RAM model of computation. In this setting, training even simple neural networks is known to be $\exists \mathbb{R}$-complete---a complexity class believed to be harder than NP, that characterizes the difficulty of solving systems of polynomial inequalities over the real numbers. 
However, this algebraic framework diverges from the reality of digital computation with finite-precision hardware.
In this work, we analyze the theoretical complexity of ERM under a realistic bit-level model ($\mathsf{ERM}_{\text{bit}}$), where network parameters and inputs are constrained to be rational numbers with polynomially bounded bit-lengths. Under this model, we reveal a sharp dichotomy in tractability governed by the network's activation function. We prove that for deep networks with {\em any} polynomial activations with rational coefficients and degree at least $2$, the bit-complexity of training is severe: deciding $\mathsf{ERM}_{\text{bit}}$ is $\#P$-Hard, hence believed to be strictly harder than NP-complete problems. Furthermore, we show that determining the sign of a single partial derivative of the empirical loss function is intractable (unlikely in BPP), and deciding a specific bit in the gradient is $\#P$-Hard. This provides a complexity-theoretic perspective for the phenomenon of exploding and vanishing gradients. In contrast, we show that for piecewise-linear activations such as ReLU, the precision requirements remain manageable: $\mathsf{ERM}_{\text{bit}}$ is contained within NP (specifically NP-complete), and standard backpropagation runs in polynomial time. Our results demonstrate that finite-precision constraints are not merely implementation details but fundamental determinants of learnability.
\end{abstract}

\newpage

\section{Introduction}

A standard theoretical formulation of empirical risk minimization (ERM) for neural
networks asks the following decision question: given a rationally specified
network architecture, a finite rational dataset, a (poly-time computable) loss, and a threshold $\gamma$,
decide whether there exists a \emph{real} parameter vector $\theta$ in these
domains whose total loss is at most~$\gamma$. In this \emph{real-valued} ERM
formulation, even for one hidden layer ERM becomes
$\exists\mathbb{R}$-complete~\cite{BertschingerHJM23,NEURIPS2021_9813b270,hankala2024complexity}:
Verifying a candidate's exact solution is as hard as deciding the existence of real
solutions to arbitrary systems of polynomial inequalities, so membership in
$\textnormal{NP}$ is unknown and widely believed to be unlikely~\cite{schaefer2024existential}.


From the perspective of actual learning systems, this formulation hides an
important issue. Modern training pipelines, and essentially all real-world
implementations, operate with finite-precision arithmetic: weights, biases,
inputs and outputs are stored as finite bit strings, and any realistically
implementable algorithm can only manipulate such encodings
\cite{goldberg1991every,higham2002accuracy}.

To better facilitate realistic settings and aiming to bridge theory and practice, we study a natural bit-model
analogue of ERM, which we call $\mathsf{ERM}_{\text{bit}}(C)$ with parameter complexity
$C=(C_1,C_2)\in\mathbb{N}^2$. An instance $I$ specifies a rationally encoded network and dataset, and the decision question is whether there exists a rational parameter vector $\theta$ for the edges such that the empirical loss $L(\theta)$ satisfies
$L(\theta)\le \gamma$, and the binary encoding length satisfies
$|\mathrm{enc}(\theta)| \le C_1\cdot |I|^{C_2}$. 
Note that every $C$ yields a different language; for simplicity, we fix a sufficiently large $C$ throughout and write
$\mathsf{ERM}_{\text{bit}}:=\mathsf{ERM}_{\text{bit}}(C)$.


Once we adopt this more concrete, bit-level formulation, the complexity of ERM
shifts already for very simple architectures. When the depth is constant, and all
activations and the loss are poly-time computable over~$\Q$, a candidate
parameter vector $\theta$ for $\mathsf{ERM}_{\text{bit}}$ has a polynomial-size
binary encoding, the empirical loss is a poly-time computable rational, and we can therefore \emph{verify} $\theta$ in
polynomial time, as in practice. Thus, shallow $\mathsf{ERM}_{\text{bit}}$ lies in $\textnormal{NP}$ in the bit sense, unlike the standard real-numbered ERM formulation. 
This contrast comes entirely from how solutions are quantified, not from the underlying Turing model.
This mismatch raises a basic question:
\begin{quote}
	\emph{To what extent do hardness results for neural networks with real numbers survive when we restrict attention to polynomially encodable rational weights?}
\end{quote}
In this paper, we study this question for \emph{deep} networks. Deep networks are of highest practical relevance \cite{DBLP:books/daglib/0040158}; however, in many theoretical aspects, they are not well understood  in contrast to their shallow counterparts. We show that the algebraic structure of the activation function dictates the bit-model complexity of training: any non-linear polynomial activation functions (with rational coefficients) can encode hard, in terms of complexity, arithmetic and force extreme exploding/vanishing gradients, whereas standard piecewise-linear activations (e.g., ReLU and leaky-ReLU) admit NP verification of $\mathsf{ERM}_{\text{bit}}$ witnesses and polynomial-time back-propagation. This distinction, supported by the preference of modern deep learning architectures for piecewise linear activations, could not be given using a real-RAM setting of ERM. 

\subsection{Main Results}

Our contributions can be summarized by the following theorems.
First, we show that for polynomial activations, $\mathsf{ERM}_{\text{bit}}$ on deep architectures can be cast as a computational hard numerical problem. 
In the next results, let $\id$ be the identity (activation) function. 
In the next result, we enforce the intended parameter behavior via auxiliary samples, yielding an $\mathsf{ERM}_{\text{bit}}$ instance that represents a hard numerical circuit for a wide family of polynomial activation functions. 
Interestingly, our next hardness results hold even for the {\em promise} version of the problem: Given $a<b \in \mathbb{N}$ and an instance $I$ of an optimization problem, decide whether the objective value of $I$ is at most $a$ or at least $b$.

\begin{theorem}
		\label{thm:SLP}
    Deciding $\mathsf{ERM}_{\text{bit}}$ for deep networks with activations in $\{\sigma,\id\}$ is $\textnormal{\#P}$-hard (under polynomial-time Turing reductions), where $\sigma \in \Q[T]$ is any non-linear polynomial activation with degree $ \ge 2$. This result holds for $0/1$ loss computable in polynomial time with (i) any fixed promise gap $a<b$, or if (ii) all intermediate node values are uniformly bounded by a constant.  
\end{theorem}

By the above $\#\mathsf{P}$-hardness, membership in $\mathsf{NP}$ would imply a collapse of the polynomial hierarchy via Toda's theorem~\cite{toda1991pp}; hence $\mathsf{ERM}_{\text{bit}}\notin\mathsf{NP}$ is widely believed.

\begin{corollary}
If $\mathsf{ERM}_{\text{bit}} \in \textnormal{NP}$, then $\textnormal{PH} \subseteq \Delta_2^P$
(in particular, the polynomial hierarchy collapses to its second~level).
\end{corollary}

This hardness does not depend on the magnitude of the numbers, as even with bounded norm in each layer, the problem remains as hard. Note that the result holds for any non-linear polynomial with rational coefficients.\footnote{if $\sigma$ is given in the input and is not universally fixed, then the intermediate values are constants that depend on $\sigma$} The above result uses a non-standard (yet efficient) loss function to obtain the $\#$-P-hardness; however, we can obtain a slightly weaker lower bound for the standard hinge loss $\ell(x,y) = \max\{0,1-y\cdot x\}$ (see the discussion for more details).

We extend this hardness to the numerical routine most central to deep learning: backpropagation. We define \textsc{Backprop-Sign} as the problem of determining the sign of a single gradient coordinate. under similar assumptions to Theorem~\ref{thm:SLP}, even deciding the \emph{sign} of a single gradient coordinate is
intractable (unlikely to have a bounded-error probabilistic polynomial time algorithm (BPP)), and this persists even when all intermediate node values are
uniformly bounded by a constant. Moreover, deciding the value of a specific bit in the gradient is $\#P$-hard, making it unlikely to even be in NP. 

\begin{theorem}
    \label{thm:backprop-hard-bit}
    Given a network with activation functions in $\{\sigma,\id\}$, for any non-linear polynomial activation $\sigma$, evaluating a single gradient coordinate at a given rational parameter $\theta$ admits the following hardness results. 
    \begin{itemize}
    \item Deciding a specific bit in the gradient is \textnormal{$\#P$-hard}. 
        \item Deciding the sign of the gradient is not in \textnormal{BPP}, assuming the Constructive Univariate Radical Conjecture~\cite{10.1145/3510359} and assuming $\textnormal{NP} \not\subseteq \textnormal{BPP}$. 
    \end{itemize}
\end{theorem}

The above result holds even for one-dimensional inputs and square loss with any fixed promise gap. Moreover, the hardness persists even when all intermediate node values are uniformly bounded by a constant. 
To the best of our knowledge, this provides the first complexity-theoretic perspective for the empirical ``exploding gradients'' and ``vanishing gradients'' phenomena~\cite{hochreiter1991untersuchungen,hochreiter1998vanishing,bengio1994learning,glorot2010understanding,pascanu2013difficulty}. In practice, exploding gradients can be controlled by various regularization techniques. However, the above result theoretically demonstrates that, under non-linear polynomial activations, bounding the norms is insufficient. 

We note that in practice ``exploding/vanishing gradients'' are a numerical
magnitude issue rather than a computational one. Our result rules out this
interpretation in the worst case: the hardness stems from a \emph{non-benign}
form of overflow/underflow, where the growth is not a simple, easily trackable
scaling. Indeed, there are benign cases---e.g., repeatedly multiplying by $2$
(or squaring powers of two) merely shifts the binary point / appends zeros in
binary, which can be tracked efficiently. In contrast, our constructions encode
general arithmetic circuits, so intermediate quantities may have succinct
descriptions but computationally hard-to-predict bit patterns.

In contrast to the above negative results, we show that standard piecewise-linear activations (such as ReLU) escape this hardness, in compliance with their abundant use in modern networks. Intuitively, unlike polynomials, they do not increase the bit-length of values multiplicatively with depth. Note that the next theorem holds even without promise, in contrast to our hardness results. This result may be known for some activations, and we give it for completeness and to highlight the contrast from polynomial activations. 

\begin{theorem}
	\label{thm:intro-relu}
	For deep networks with piecewise-linear activations, $\mathsf{ERM}_{\textnormal{bit}}$ is $\textnormal{NP}$-\textnormal{complete}. Furthermore, verifying a witness and computing the exact gradient (one step of back-propagation) can be done in polynomial time.
\end{theorem}

The above result also holds for rounding non-linear activations to a fixed number of bits.

\begin{figure}[H]
	\centering
\includegraphics[width=0.8\columnwidth]{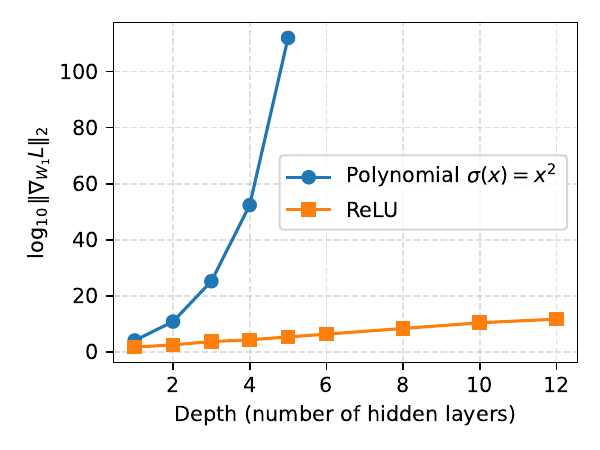}
	\caption{Illustration of Depth-dependent gradient growth for polynomial vs.\ ReLU
		activations. We plot $\log_{10}\|\nabla_{W_1} L\|_2$ versus depth for
		randomly initialized MLPs with $\sigma(x)=x^2$ or ReLU, where in both
		cases weights are scaled by a factor of~3 on every layer. Gradients
		increase with depth for both activations, but much faster for the
		polynomial one, qualitatively reflecting our bit-complexity separation. This experiment is illustrative rather than a formal bound.}
	\label{fig:gradients-poly-vs-relu}
\end{figure}

\begin{remark}[Bit-bounded activations: polynomial-time evaluation and backpropagation]
\label{cor:bit-bounded}
Fix $k\ge 1$ with $k\le \poly(|I|)$, where $|I|$ is the encoding size of the instance. Suppose every activation $\sigma^v$ is some function
$\phi_k:\Q\to\Q$ that (i) is computable in time $\poly(|I|)$ on rational inputs, (ii) always outputs $k$-bit
dyadic rationals (optionally after clipping to a fixed interval), and (iii) is equipped with a specified
backprop rule/derivative oracle that is also computable in time $\poly(|I|)$ and returns rationals of bit-length
$\poly(|I|)$. Then, one forward pass and one backward pass (computing any fixed gradient coordinate) run in
polynomial time in the bit model.
\end{remark}

This result extends to any architecture whose forward/backward computation is
given by an explicit polynomial-size computation DAG in our bit model.
In particular, it holds for RNN architectures run for $T\le \textnormal{poly}(N)$ steps, by
unrolling the recurrent computation into a depth-$T$ DAG and applying the same
argument.
Table~\ref{tab:main-summary} summarizes these complexity comparisons. In addition,
\begin{table}[h]
	\centering
	\small
	\setlength{\tabcolsep}{4pt}
	\begin{tabular}{@{}lcc@{}}
			\toprule
			\textbf{Setting} & \textbf{Real-valued ERM} & \textbf{$\mathsf{ERM}_{\text{bit}}$} \\
			\midrule
			Shallow ERM & $\notin \textnormal{NP}^{\ast\ast}$ & NP-complete \\
			Deep ERM, poly $\sigma$ & $\cdot$ & $\notin \textnormal{NP}^{\ast}$ \\
			Deep ERM, pw-linear $\sigma$ &  $\notin \textnormal{NP}^{\ast\ast}$ & NP-complete \\
			Backprop-Sign, poly $\sigma$ & in $\textnormal{P}$  & $\notin \textnormal{BPP}^{\ast\ast\ast}$ \\
			Backprop / sign, pw-linear $\sigma$ & $\cdot$ & in $\textnormal{P}$ \\
			\bottomrule
		\end{tabular}
	\caption{Main complexity comparisons between the classical real-valued ERM
			formulation and the bit-bounded $\mathsf{ERM}_{\text{bit}}$ formulation.
			``poly $\sigma$'' / ``pw-linear $\sigma$'' denote fixed polynomial /
			piecewise-linear (ReLU-type) activations; The entry $\notin \textnormal{NP}^{\ast\ast}$ indicates hardness via
			the existential theory of the reals ($\exists\mathbb{R}$), so NP membership
			is unknown and believed unlikely for real-valued ERM. The entry marked
			${}^{\ast}$ assumes that the polynomial hierarchy (PH) does not collapse $\textnormal{PH} \not \subseteq \Delta^P_2$, and			${}^{\ast\ast\ast}$ assume the constructive univariate radical conjecture and $\textnormal{NP}\not\subseteq\textnormal{BPP}$.}
	\label{tab:main-summary}
\end{table}
We give an intuitive visualization of the growth rate difference between non-linear polynomials and piecewise linear activations in \Cref{fig:gradients-poly-vs-relu}. Our result offers a complexity-theoretic perspective that may help explain, at an abstract level, why piecewise-linear activations are attractive in modern deep learning architectures.

\subsection{Discussion}
This section discusses the implications and limitations of our results and outlines several open questions.

\paragraph*{Real vs. bit ERM.}
Our results show that the choice of computational model is central to understanding the complexity of deep learning. Namely, unrealistic real numbers render the complexity of ERM far from ERM on numbers that can be represented in practice, as shown in \Cref{tab:main-summary}. We believe that our $\mathsf{ERM}_{\text{bit}}$ is closer to ERM in practice and consider deriving further bounds in this model as an interesting theoretical direction with practical significance. 
It is worth noting that  the $\mathsf{ERM}_{\text{bit}}$ model is still somewhat far from practice and should still be considered as an intermediate model that is closer to practice. 
For example, our positive result for piecewise linear activations for $\mathsf{ERM}_{\text{bit}}$ and Backprop-Sign overlook the empirical fact that gradient exploding can still occur in such activations \cite{DBLP:books/daglib/0040158}; however, such gradient explosions are incomparably small than those produced by non-linear polynomials.

\paragraph*{Depth, activation functions, and practical relevance.}
For deep networks, we obtain a sharp contrast between polynomial activations and standard piecewise-linear activations such as ReLU, highlighting depth and the activation as key drivers of computational hardness in the bit model. 
Our lower bounds apply for {\em any} non-linear polynomial with rational coefficients. Polynomial activations, especially quadratic ones, are well studied both in practice and theoretically~\cite{soltani2019fast, DBLP:conf/colt/Allen-ZhuL23, blondel2017multi, livni2014computational, sarao2020optimization, bartan2021neural, du2018power, soltanolkotabi2018theoretical}. In addition, due to Taylor expansion, used to approximate smooth activations (sigmoid, tanh, etc) \cite{roheda2024volterra}, polynomial activations admit strong practical motivation. Hence, even though our results are worst-case and theoretical, we believe our dichotomy gives a valuable distinction between the two families (piecewise linear and polynomial).

Beyond the polynomial versus piecewise-linear regimes, it is natural to ask where other nonlinearities used in practice fit into this picture, such as smooth activations or non-Lipschitz activations like \(1/x\). Our current techniques do not classify these activations, and accomplishing this remains an open question. In particular, we show in Appendix~\ref{sec:lower_bound_poly} that our techniques cannot be generalized to power functions with non integer rationals or negative integer exponents. 

\paragraph*{Exploding/vanishing gradients and ``non-benign'' overflow.}
Our lower bounds should not be read as claiming that practical exploding or
vanishing gradients are \emph{typically} computationally hard. Rather, they show
that in finite-precision models there exist worst-case instances where overflow
(or underflow) is \emph{algorithmically} non-benign: the issue is not merely that
numbers become large, but that their binary representations encode hard
arithmetic; thus, ``tracking the magnitude'' is insufficient. This contrasts with
benign overflow patterns (e.g., repeated exponentiation with base $2$) that correspond
to simple bit shifts and are easier to predict and control.

\paragraph{Loss Functions} The $\#$-P-hardness result presented in \Cref{thm:SLP} uses a non-standard loss function. This loss function is used to efficiently extract a specific bit $j$ from a very large number; this cannot be done using standard losses such as squared, hinge, etc., due to their intrinsic algebraic properties. Even though,  
we still can obtain a slightly weaker lower bound using the standard hinge loss $\ell(x,y) = \max\{0,1-y\cdot x\}$ by considering the sign (rather than a specific bit) of the output of an arithmetic circuit. We give the details in Appendix~\ref{sec:omited_3}.

\paragraph{Dimensions and Regularization}
The main hardness result in \Cref{thm:SLP} uses large input/output dimensions. A natural open
question is whether the same \(\#P\)-hardness can be shown while keeping both dimensions \(O(1)\). We note that if one allows a centered
quadratic regularizer \(\lambda\|\theta-\theta^\ast\|_2^2\), then such a
constant-dimensional variant is possible as we show in the appendix (see
Remark~\ref{rem:one-dim-via-l2}).

\paragraph{Number representation of deep networks in practice.}
Real-world deep learning does not operate over exact reals: parameters, activations,
and gradients are stored in finite-precision floating-point formats (typically IEEE~754),
with rounding, a finite exponent range, and special values such as $\pm\infty$ and NaN;
overflow and underflow are therefore intrinsic to the arithmetic~\cite{ieee7542019}.
Moreover, modern training commonly uses \emph{mixed precision}: many tensors are kept in
low precision (e.g., FP16/BF16/TF32) for throughput, while critical accumulations
(e.g., master weights and optimizer updates) are performed in FP32; in FP16, (dynamic)
loss scaling is used to mitigate gradient underflow~\cite{micikevicius2017mixed,kalamkar2019bfloat16,nvidia_a100_2020}.
Our model abstracts these representation choices while still respecting the finite-precision
regime, which is invisible in real-RAM abstractions.

\paragraph{Smooth activations in practice.}
Smooth nonlinearities (sigmoid, $\tanh$, softmax, GELU) are typically implemented via finite-precision routines
(e.g., polynomial/rational approximations or lookup tables) followed by rounding (and often clipping or numerical
stabilization). In our terminology, such implementations correspond to bit-bounded activation maps $\phi_k$ with
$k$-bit outputs. Thus, for any $k\le \poly(N)$, Remark~\ref{cor:bit-bounded} implies polynomial-time forward
evaluation and one backward pass for the resulting computation DAG. This remark concerns evaluation of a fixed
finite-precision implementation, and does not assert tractability for the corresponding exact real-valued
activation.

\paragraph*{Real-RAM vs. bit model in other domains}
Our results give evidence that learning theory models are occasionally far away from practice. Studying core learning theory questions on more realistic models, thus bridging theory and practice, is a core challenge. For example, we give in Appendix~\ref{sec:PAC} an easy construction showing that in the setting of PAC learning, the question of how many samples are needed to learn a simple class in the realizable setting, depends on whether one works in an ideal real-number model or in a finite-precision bit model. 
\begin{theorem}
	\label{thm:intro-pac}
	There exists a hypothesis class and distribution where the realizable PAC sample complexity is $O(1)$ in the exact \textnormal{real-RAM} model, but scales as $\Omega(q)$ in the bit model with $q$-bit precision rounding.
\end{theorem}

\paragraph{Quantization and finite-precision training.}
The bit-bounded formulation $\mathsf{ERM}_{\text{bit}}$ directly models
training under a fixed global precision constraint, as in fixed-point or
low-bit floating-point arithmetic (e.g., 8-bit or 16-bit formats) used on
modern hardware. The restriction that witnesses have polynomially bounded
bit-length corresponds to requiring that all weights, biases, and
intermediate activations be stored in a finite-precision format.

 Our $\mathsf{ERM}_{\text{bit}}$ model can be viewed as {\em quantization}  of the standard ERM theoretical problem on real numbers. In practice, there are several reasons to perform quantization \emph{during} 
training, as in our model: (i) memory and bandwidth
constraints on accelerators, (ii) the need to simulate the numerical behavior
of low-precision deployment hardware, and (iii) empirical evidence that
quantization-aware training can improve robustness and generalization for
quantized models \cite{courbariaux2015binaryconnect,rastegari2016xnor,
	polino2018model,rokh2023comprehensive}. Our bit-model results thus provide
worst-case complexity guarantees tailored to such finite-precision training
regimes \cite{tailor2020degree,zhou2018adaptive,nagel2022overcoming,sakr2022optimal,liu2023llm}, rather than to a post hoc ``train-then-quantize'' pipeline.

\subsection{Related Work}

\paragraph{Real-RAM computation and $\exists\mathbb{R}$-hardness.}
The Blum–Shub–Smale (BSS) model~\cite{blum1989theory} and Smale’s program on complexity over $\mathbb{R}$~\cite{smale1998mathematical} provide the standard real-RAM abstraction allowing each basic operation on real numbers to take $O(1)$. 
Many geometric and algebraic decision problems can be expressed in the \emph{existential theory of the reals} ($\exists\mathbb{R}$), which asks whether a given system of polynomial equalities and inequalities has a real solution~\cite{DBLP:conf/gd/Schaefer09}. 
A problem is called \emph{$\exists\mathbb{R}$-hard} if it is at least as hard as any problem in this class. It is known that $\textnormal{NP} \subseteq \exists \mathbb{R} \subseteq \textnormal{PSPACE}$ and widely assumed that both containments are proper \cite{canny1988some,schnider2023topological,schaefer2024existential}. Unlike NP, an $\exists\mathbb{R}$ witness is a tuple of \emph{real} numbers, which
need not admit a succinct rational encoding; this is one reason $\exists\mathbb{R}$
is believed to strictly contain NP \cite{DBLP:conf/gd/Schaefer09,schaefer2024existential}. Hence, classifying a problem as $\exists \mathbb{R}$-complete is considered a stronger result compared to merely NP-completeness. In particular,
in the standard ERM framework allowing infinite precision, training even fairly simple neural networks with one hidden layer is $\exists\mathbb{R}$-complete~\cite{NEURIPS2021_9813b270,BertschingerHJM23,hankala2024complexity}, where the completeness part assumes the real-RAM model.  

\paragraph{Bit-model hardness and SLP.}
For our hardness results, we rely on the \textsc{PosSLP} and \textsc{BitSLP} problems (see~\cite{DBLP:journals/siamcomp/AllenderBKM09, koiran2011interpolation,DBLP:conf/fsttcs/BlaserDJ24}), the problem of deciding the sign or a specific bit of a deep arithmetic circuit over the operations $\{\times,+,-\}$ and constants $0,1$. Bürgisser and Jindal~\cite{DBLP:conf/soda/BurgisserJ24} provided conditional lower bounds for \textsc{PosSLP} under the constructive radical conjecture of~~\cite{10.1145/3510359}. We build on these works to obtain hardness for deep-network $\mathsf{ERM}_{\text{bit}}$ and for \textsc{Backprop-Sign} with polynomial activations.

\paragraph{Hardness and algorithms for ERM.}

The NP-hardness of training small or shallow networks in the bit model has been known since the works of \cite{judd1988complexity,blum1988training,vsima2002training}. For ReLU and other piecewise-linear activations, \cite{dey20} analyzed one-node ReLU training, \cite{Goel21} established sharp hardness thresholds for depth-2 networks, and \cite{Froese22,Froese2024,froese2022computational} proved parameterized and fixed-dimension hardness results; \cite{Boob22} obtained further NP-hardness and approximation results. 

Several algorithms were also given to ERM. Important examples include the works of \cite{AroraBMM18,PilanciE20,ergen2021convex,ergen2021global}, who studied ReLU networks and derived polynomial-time algorithms under structural restrictions (e.g., fixed input dimension or width). For linear-threshold activations, \cite{khalife2024neural} give algorithms polynomial in the dataset size but exponential in network size and ambient dimension.

The recent work of \cite{doron2025hardness} provides a lower bound for {\em discrete} ERM, where each parameter weight is chosen from a discrete finite set given in the input explicitly. For this model, the above paper shows that this discrete variant of ERM is computationally hard. 
While the setting bears similarities to our first result, the above model is much more restrictive than our general bit-level ERM. More importantly, the activation functions used by \cite{doron2025hardness} are highly non-continuous and specifically engineered, in sharp contrast to our general family containing any non-linear polynomial of rational coefficients.

\paragraph{Exploding/vanishing gradients.}
Exploding and vanishing gradients have been studied extensively in recurrent and deep networks~\cite{hochreiter1991untersuchungen,hochreiter1998vanishing,bengio1994learning,glorot2010understanding,pascanu2013difficulty}, and more recent work analyzes instabilities in deep training dynamics~\cite{sun2022surprising} and surveys these issues~\cite{DBLP:books/daglib/0040158}. Our contributions are orthogonal: we construct explicit families of networks with polynomial activations whose gradients tend to be extremely large or small in both bit-complexity and magnitude sense, and we show that even deciding their sign is intractable under standard complexity assumptions. In terms of computational complexity, the work of~\cite{backurs2017fine} provides a fine-grained (polynomial) lower bound for computing the gradient. 

\paragraph{Organization.}
\Cref{sec:prel} formalizes our bit model, $\mathsf{ERM}_{\text{bit}}$, and the straight-line program (SLP) primitives.
\Cref{sec:ERM} proves the deep-network ERM dichotomy, including the $\#\mathsf{P}$-hardness for polynomial activations (\Cref{thm:SLP}). 
\Cref{sec:backprop} establishes the corresponding hardness for backpropagation (\Cref{thm:backprop-hard-bit}).
Due to space constraints, additional background and deferred proofs appear in the appendix.

\section{Preliminaries}
\label{sec:prel}

\paragraph{Notations} Let $\mathbf{1}_{B}$ be the indicator function for some boolean condition $B$, which is equal to $1$ if $B$ is true and is set to $0$ otherwise; for some $k \in \mathbb{N}$, let $[k] = \{1,\ldots,k\}$.

\paragraph{Bit Model and Rational Encoding}
We work in the standard Turing machine (bit) model, where generally numbers are represented in binary and running time is measured in bit operations. A rational $q\in\mathbb{Q}$ is encoded as a reduced pair $(\mathrm{num}(q),\mathrm{den}(q))\in\mathbb{Z}\times\mathbb{N}$ with $\mathrm{den}(q)>0$ and $\gcd(\mathrm{num}(q),\mathrm{den}(q))=1$, and its size is the total bit-length of $\mathrm{num}(q)$ and $\mathrm{den}(q)$. We write $|I|$ for the total bit-length of a standard encoding of an input instance $I$, and say an algorithm runs in polynomial time if it uses ${|I|}^{O(1)}$ bit operations. Arithmetic over $\mathbb{Q}$ is exact (addition, multiplication, and polynomial evaluation), with cost measured by bit length; unlike the real-RAM model, operations are not unit-cost.

\paragraph{Empirical Risk Minimization}
We give below the formal definition of ERM in the bit model. 
\paragraph{Input.}
An ERM instance specifies a feedforward network and a finite dataset.
The network is a directed acyclic graph $N=(V,E)$ with vertex set
$V=S \cup H\cup T$, where $S$ are the unique input vertices, $H$ is the set of
hidden neurons, and $T$ is the set of output neurons.
 The \emph{depth} is the length of
the longest directed path from some $s \in S$ to a vertex in $T$.
The dataset is
$
\cD = \{(x_i,y_i)\}_{i=1}^n,
$
with inputs $x_i \in \cX \subseteq \mathbb{Q}^{S}$ and labels
$y_i \in \cY \subseteq \mathbb{Q}^{|T|}$.
Each $v\in H\cup T$ has a fixed activation
$\sigma^v:\mathbb{Q}\to\mathbb{Q}$, and we are given a loss function
$\cL:\cY\times\cY\to\mathbb{Q}$. All $\sigma^v$ and $\cL$ are
polynomial-time computable on rational inputs.

let $\Theta = \Q^{E} \times \Q^{E}$ be the parameter domain of the instance, which enables selecting scalar weight and bias per edge.  
All numerical data in the instance are rationals encoded as specified above; the \emph{size} of the instance is the total bit-length of this
encoding.

\paragraph{Network semantics.}
A parameter vector is a tuple
\[
\theta=(w_{u,v},b_{u,v})_{(u,v)\in E}\in\Theta,\footnote{Our lower bounds yield similar results if a  bias is defined for every node instead for every~edge.}
\]
Given $\theta$ and an input $x\in\Q^{S}$, for $s \in S$ which is a source in the network, set $f^s_\theta(x):=x_s$. For each other $v \in V$ define
\[
f^v_\theta(x)
:=\sigma^v\!\left(x_v+
\sum_{(u,v)\in E} \bigl(w_{u,v}\cdot f^u_\theta(x) + b_{u,v}\bigr)
\right),
\]
where $x_v := 0$ if $v \notin S$.  Define the network output as $f_\theta(x):=(f^t_\theta(x))_{t\in T}\in\cY$. 

\paragraph{Objective and decision problem.}
The empirical loss of the parameter set $\theta$ on $\cD$ is defined by
$$
L(\theta) := \sum_{i=1}^n \cL\bigl(f_\theta(x_i), y_i\bigr).
$$
For some $C = (C_1,C_2) \in \mathbb{N}^2$, in the bit-model $\mathsf{ERM}_{\text{bit}}(C)$ decision problem we consider the language
\begin{equation*}
	\begin{aligned}
		\mathsf{ERM}_{\text{bit}}(C)
		&= \bigl\{\, I \,\big|\, \exists \theta\in \Theta
		\text{ s.t. } L(\theta)\le\gamma,\\
		&\qquad |\mathrm{enc}(\theta)| \le C_1 \cdot |I|^{C_2} \bigr\}.
	\end{aligned}
\end{equation*}
where $I$ is an encoded ERM instance with binary bit-length $|I|$, and
$\mathrm{enc}(\theta)$ is the binary encoding of $\theta$
(the certificate) with encoding size $|\mathrm{enc}(\theta)|$. Recall that we use $\mathsf{ERM}_{\text{bit}} = \mathsf{ERM}_{\text{bit}}(C)$ for a fixed and sufficiently large parameter complexity $C$ throughout the paper.

In other words, given an instance $I=(N,\cD,\cL,\gamma)$, the bit-model ERM decision problem asks whether there exists a
\emph{rational} parameter vector $\theta$ whose binary encoding has length polynomial in $|I|$ such that
$L(\theta)\le\gamma$. This excludes solutions with super-polynomial encoding, which are allowed in real-RAM.

\begin{definition}[$(a,b)$-promise $\mathsf{ERM}_{\text{bit}}$]
\label{def:ermbit-promise}
Fix $a<b$ in $\mathbb{Q}$. In the \emph{$(a,b)$-promise} version, the input is an $\mathsf{ERM}_{\text{bit}}$ instance
$I$ (with the same encoding conventions) and it is promised that either
(i) there exists a feasible $\theta$ with $L(\theta)\le a$ (YES case), or
(ii) for every feasible $\theta$ we have $L(\theta)\ge b$ (NO case).
The goal is to distinguish YES from NO.
\end{definition}

\paragraph{Complexity Classes.} 
The class $\mathrm{P}$ consists of all problems that can be solved in polynomial time by such a deterministic Turing machine. The class $\mathrm{NP}$ consists of all problems for which a proposed solution can be verified in polynomial time by a deterministic Turing machine, given the input and the solution. 
{\em Bounded-error probabilistic polynomial time (BPP)} describes decision problems decidable by a probabilistic Turing machine with success probability at least $\frac{2}{3}$ in polynomial time. The class $\#\mathrm{P}$~\cite{valiant1979complexity} is the counting analogue of $\mathrm{NP}$: it contains problems that ask for the number of polynomial-length certificates (witnesses) that make a polynomial-time verifier accept a given input.
While P $\subseteq$ BPP and P $\subseteq$ NP, it is only conjectured that P $=$ BPP and P $\neq$ NP.

\paragraph{Straight Line Programs (SLP)}
The paper will use {\em straight-line programs (SLP)} in some of the reductions. 

\begin{definition}
    \label{def:SLP_p}
	A (numeric) {\em straight-line program (SLP)} $P$ is a sequence of rationals $(a_0, a_1, \ldots , a_{\ell})$ such
	that $a_0 \in \Q$ and $a_i = a_j \oplus a_k$~ for all $1 \leq i \leq \ell$, where $\oplus \in \{+, -, *\}$ and $j, k < i$. 
\end{definition} 

The definition of SLP given above is a softened definition defined for our reduction needs. Specifically, SLPs are commonly defined as univariate polynomials with $a_0 = 1$. In our lower bounds, we use either $a_0 = 1$ or $a_0 = 2^{-\ell}$, where the latter yields programs with all intermediate values in $[-1,1]$.  
A thorough discussion and related results proving why the above definition is without loss of generality in our setting, and further background on the subject are provided in Appendix~\ref{sec:SLP}. 

In SLPs, we are given only the inductive representation of their output (i.e., the last variable). Importantly, we do not have an explicit encoding of the variable's values.

This leads to the following problems originating in \cite{DBLP:journals/siamcomp/AllenderBKM09} and also appears in \cite{koiran2011interpolation,DBLP:conf/fsttcs/BlaserDJ24}. 

\begin{problem}
	\label{prob:BitSLP_p}
	\textnormal{\textsf{BitSLP:}} Given an integer $j \in \mathbb{N}$ and an \textnormal{SLP} $P = (a_0,\ldots,a_{\ell})$ computing a rational $u/v = a_{\ell} = n_P$ with $u \in \mathbb{Z}$ and $v \in \mathbb{N}$, decide if the $j$-th \textnormal{LSB} binary digit of $n_P$ is $1$, i.e., decide whether $\lfloor \tfrac{|u|}{v} / 2^j \rfloor \bmod 2 = 1$. 
\end{problem}

A strong hardness result for this problem follows from~\cite{DBLP:journals/siamcomp/AllenderBKM09}. 

\begin{theorem}
	\label{thm:BitSLP}
\cite{DBLP:journals/siamcomp/AllenderBKM09}: 
\textnormal{\textsf{BitSLP}} is $\textnormal{\#P-hard}$. 
\end{theorem}

We also use the following slightly easier version, concerning the sign of $n_P$ rather than a specific bit, as a source of hardness for our reductions. 
\begin{problem}
	\label{prob:SLP}
	\textnormal{\textsf{PosSLP:}} Given an \textnormal{SLP} $P = (a_0,\ldots,a_{\ell})$ computing a number $n_P = a_{\ell}$, decide if $n_P > 0$. 
\end{problem}

We will use a lower bound on \textnormal{PosSLP} obtained by \cite{DBLP:conf/soda/BurgisserJ24}, stating that a conjecture of \cite{10.1145/3510359} (see Appendix~\ref{sec:SLP}) and a randomized polynomial time algorithm for PosSLP would imply the very unlikely statement that \textnormal{NP} $\subseteq$ \textnormal{BPP}.  

\begin{theorem}
	\label{thm:PosSLP'}
	\cite{DBLP:conf/soda/BurgisserJ24}: If \textnormal{Conjecture~\ref{con:radical}} is true and \textnormal{PosSLP} $\in$ \textnormal{BPP}, then \textnormal{NP} $\subseteq$ \textnormal{BPP}.
\end{theorem}

\begin{figure*}[t]
\centering
\resizebox{\textwidth}{!}{%
\begin{tikzpicture}[
  box/.style={draw,rounded corners,align=center,inner sep=5pt},
  arr/.style={-latex,thick},
  note/.style={align=left,font=\scriptsize},
  node distance=10mm
]
\node[box] (slp) {SLP};
\node[box, right=22mm of slp] (gadget) {$\sigma$-gadget\\simulates $+,-,\times$};
\node[box, right=22mm of gadget] (ermr) {Hardness of $\mathsf{ERM}_{\text{bit}}$\\(promise gap)};

\node[box, right=22mm of ermr] (back_b) {Backprop-Bit $\#P$-hard};

\node[box, below=14mm of ermr] (back) {Backprop-Sign\\not in BPP (cond.)};

\draw[arr] (slp) -- node[note,above]{Lemma~\ref{lem:polynomial_representation}} (gadget);
\draw[arr] (gadget) -- node[note,above]{Lemma~\ref{lem:SLP-to-ERM}} (ermr);
\draw[arr] (ermr) -- (back);
\draw[arr] (ermr) -- (back_b);

\node[box, below=14mm of slp, minimum width=60mm] (relu) {Piecewise-linear\\$\mathsf{ERM}_{\text{bit}}\in$NP; exact backprop in P};
\end{tikzpicture}%
}
\caption{Proof roadmap: SLP based hardness is embedded into deep networks via a fixed polynomial activation $\sigma$, yielding conditional lower bounds for $\mathsf{ERM}_{\text{bit}}$ and Backprop-Sign; piecewise-linear activations admit NP verification and polynomial-time exact backprop in the bit model.}
\label{fig:roadmap}
\end{figure*}
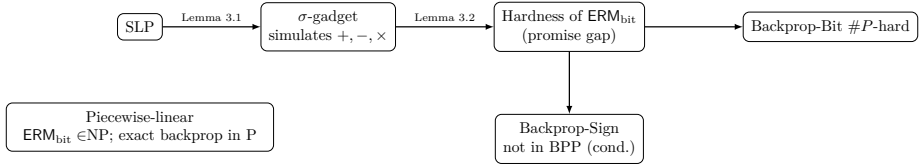

\section{Computational Complexity of Deep Networks}
\label{sec:ERM}
In this section, we prove \Cref{thm:SLP}. 
We use the following result in the construction. Let $\mathbb{F}[T]$ denote the ring of polynomials in the indeterminate $T$ with coefficients in a field $\mathbb{F}$. 

\begin{lemma}
	\label{lem:polynomial_representation}
	Let
	$f(T) = \sum_{k=0}^{\mu} r_k T^k \in \mathbb{Q}[T]$, where $r_{\mu} \neq 0,  {\mu} \geq 2
	$,
	be a non‑linear polynomial with rational coefficients. Then, there is a polynomial-time algorithm that, given $f$, returns rational numbers $\lambda_0,\ldots,\lambda_{\mu}$ such that for every $x,y \in \mathbb{R}$ it holds that 
	\[\sum_{j=0}^{\mu} \lambda_j\bigl(f(x+y+j) - f(x+j) - f(y+j)\bigr) = x \cdot y.\]
\end{lemma}

For example, if $f(x) = x^2$, then $\lambda_0 = \frac{1}{2}$ and $\lambda_1,\lambda_2 = 0$ satisfy the conditions of the lemma since $\frac{1}{2}((x+y)^2-x^2-y^2) = x \cdot y$.
	All coefficients $r_k$ of $f$ and all coefficients $\lambda_j$
are in $\mathbb{Q}$ and are not restricted to be integers. 
Note also that ``polynomial'' here means with integer exponents. Allowing non‑integer exponents (such as $T^{1/2}$) would break the argument. We show in Appendix~\ref{sec:lower_bound_poly} that the statement cannot be generalized for rational non-integer exponents and for integer negative exponents.

We give the following reduction from BitSLP to $\mathsf{ERM}_{\text{bit}}$, to be used in the proof of \Cref{thm:SLP}. We give in \Cref{fig:roadmap} a road map for our SLP reductions and illustrate the construction in \Cref{fig:multiplication-gadget}.  

\begin{lemma}
	\label{lem:SLP-to-ERM}
	Let $\sigma\in\mathbb{Q}[T]$ be a non-linear polynomial and let $a,b \in \mathbb{N}$ such that $a<b$. There is a
	polynomial-time procedure that, given a \textnormal{BitSLP} instance
	$(j,P=(a_0,\dots,a_\ell))$, constructs an $\mathsf{ERM}_{\text{bit}}$ instance
	$I(P)$ with promise gap $a<b$ and activations in $\{\sigma,\mathrm{id}\}$ 
such that there is $\theta^\ast(P) \in \Theta$ satisfying
		$
		f_{\theta^\ast(P)}(a_0) \;=\; a_\ell \;=:\; n_P
		$.
		 In addition, $P$ is a \textnormal{YES}-instance of \textnormal{BitSLP} if
		and only if $I(P)$ is a \textnormal{YES}-instance of $\mathsf{ERM}_{\text{bit}}$ (with a promise $a<b$).
	The size of $I(P)$ is
	polynomial in the length of~$P$.
\end{lemma}

Using the above lemma, the proof of \Cref{thm:SLP} follows. 

\begin{figure}[t]
	\centering
	\resizebox{0.9\columnwidth}{!}{%
		\begin{tikzpicture}[
			font=\scriptsize,
			>=latex,
			var/.style={circle,draw,minimum size=6mm,inner sep=0pt},
			sigma/.style={circle,draw,minimum size=6mm,inner sep=0pt,fill=gray!10},
			lin/.style={rectangle,draw,rounded corners,minimum height=6mm,
				minimum width=10mm,align=center},
			]
			
			\node[var] (x) at (0,1.0) {$x$};
			\node[var] (y) at (0,-1.0) {$y$};
			
			\node[lin] (sxy) at (2,1.0) {$x+y+j$};
			\node[lin] (sx)  at (2,0.0) {$x+j$};
			\node[lin] (sy)  at (2,-1.0) {$y+j$};
			
			\node[sigma] (sig1) at (4,1.0) {$\sigma$};
			\node[sigma] (sig2) at (4,0.0) {$\sigma$};
			\node[sigma] (sig3) at (4,-1.0) {$\sigma$};
			
			\node[lin] (lincomb) at (6,0)
			{linear\\combination\\with $\lambda_j$};
			
			\node[var,right=0.8cm of lincomb] (out) {$x\cdot y$};
			
			\draw[->] (x) -- (sxy);
			\draw[->] (y) -- (sxy);
			\draw[->] (x) -- (sx);
			\draw[->] (y) -- (sy);
			
			\draw[->] (sxy) -- (sig1);
			\draw[->] (sx)  -- (sig2);
			\draw[->] (sy)  -- (sig3);
			
			\draw[->] (sig1) -- ++(0.7,0) |- (lincomb.north);
			\draw[->] (sig2) -- (lincomb.west);
			\draw[->] (sig3) -- ++(0.7,0) |- (lincomb.south);
			
			\draw[->] (lincomb) -- (out);
			
			\node[draw,dashed,inner sep=6pt,fit=(sxy) (sx) (sy) (sig1) (sig2) (sig3)]
			(block) {};
			
			\node[above=0.1cm of block.north] {for each shift $j=0,\dots,{\mu}$};
			
		\end{tikzpicture}%
	}
	\caption{Multiplication gadget using a fixed polynomial activation
		$\sigma\in\mathbb{Q}[T]$. For each integer shift $j=0,\dots,{\mu}$, the network
		forms three affine combinations $x+y+j$, $x+j$, and $y+j$, applies $\sigma$,
		and then takes a fixed linear combination with coefficients
		$\lambda_j\in\mathbb{Q}$ (indicated by the ``linear combination'' node).}
	\label{fig:multiplication-gadget}
\end{figure}
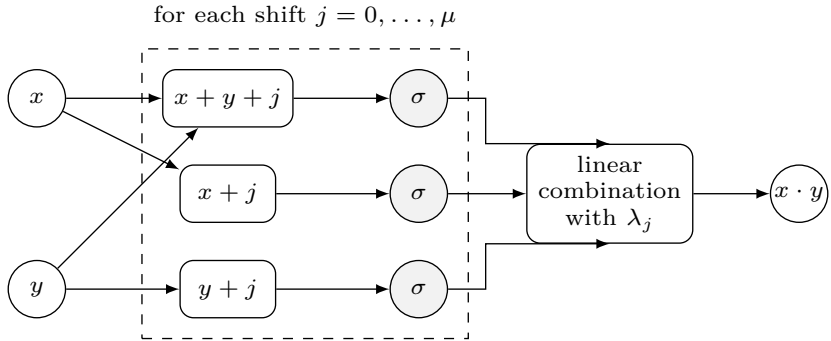

\section{Computational Complexity of Gradient Descent}
\label{sec:backprop}
In this section, we show that it is unlikely that there is a randomized polynomial-time
algorithm that, given a rationally specified network with polynomial activation, a dataset, a loss function, and a
parameter vector, decides the sign of a single gradient coordinate. 

We work in the same bit model and with the same network
and loss conventions as before, but we restrict to scalar-input,
scalar-output networks and the activation family introduced in the
previous section. 
Concretely, we assume:
(i) the input space is $\cX\subseteq\mathbb{Q}$ and the label space is
$\cY\subseteq\mathbb{Q}$;
(ii) for every non-input neuron $v$ the activation is either
$\sigma^v=\sigma$ or $\sigma^v=\mathrm{id}$, where
$\sigma(T)=\sum_{k=0}^{\mu} r_k T^k\in\mathbb{Q}[T]$ is a fixed non-linear
polynomial with rational coefficients;
(iii) the loss $\cL:\mathbb{Q}\times\mathbb{Q}\to\mathbb{Q}$ is
polynomial-time computable and differentiable in its first argument, and
the partial derivative with respect to the first argument is
polynomial-time computable and rational-valued on rational inputs.

We regard the loss $L(\theta)=\sum_i \cL(f_\theta(x_i),y_i)$, which is
defined by an algebraic expression in the parameters, as the restriction
to $\mathbb{Q}^p$ of a real-analytic function $L:\mathbb{R}^p\to\mathbb{R}$.
For any scalar parameter coordinate $w_e$, we write
$\partial L/\partial w_e(\theta)$ for the usual real partial derivative
of this extension evaluated at the given rational point $\theta$.
 The Backprop-Sign problem asks for the sign of a single gradient coordinate at
 a given rational parameter point.  The Backprop-Bit problem asks for the value of the $j$-th lower bit in this gradient. We provide a concise definition of both problems below.

\begin{definition}[Backpropagation sign  and Backpropagation bit problems]
	\label{def:backprop-sign}
	An instance of \textnormal{Backprop-Sign} with polynomial activation is a tuple
	$
	(N,\cD,\cL,\theta,e^\star)
	$
    (respectively, for \textnormal{Backprop-Bit} an instance is $(N,\cD,\cL,\theta,e^\star,j)$ with $j \in \mathbb{N}$)
	where $N=(V,E)$ is a network as above with activations
	$\sigma^v\in\{\sigma,\mathrm{id}\}$ for all non-input vertices
	$v\in V$ where $\sigma \in \Q[T]$ is a non-linear polynomial activation with degree $ \mu \geq 2$, $\cD=\{(x_i,y_i)\}_{i=1}^n\subseteq\mathbb{Q}\times\mathbb{Q}$
	is a dataset, $\cL:\mathbb{Q}\times\mathbb{Q}\to\mathbb{Q}$ is a
	polynomial-time computable loss, $\theta=(w_e,b_e)_{e\in E}$ is a
	parameter vector with all $w_e,b_e\in\mathbb{Q}$, and
	$e^\star\in E$ is a distinguished edge.  The question is whether
	$
	\frac{\partial L}{\partial w_{e^\star}}(\theta) > 0,
	$
    (whether the $j$-th bit in $\frac{\partial L}{\partial w_{e^\star}}(\theta)$ is $1$)
	where $L(\theta)=\sum_{i=1}^n \cL(f_\theta(x_i),y_i)$ is the empirical
	loss.  The input size is the total bit-length of all rationals
	specifying $N,\cD,\cL$ (as a program), $\theta$, and $e^\star$.
\end{definition}

The above formulation for \textnormal{Backprop-Sign} can be made computationally easier with a promise gap $a < b$, where we have to decide if the partial derivative is larger than $b$ or smaller than $a$. Formally, 

\begin{definition}[\textnormal{Backprop-Sign} with a promise]
	\label{def:backprop-sign-promise}
	An instance of \textnormal{Backprop-Sign} with a promise gap is a tuple 
	$
	(N,\cD,\cL,\theta,e^\star,b),
	$
	where $(N,\cD,\cL,\theta,e^\star)$ is a \textnormal{Backprop-Sign} instance and $b \in \mathbb{N}$. The goal is to distinguish between $\frac{\partial L}{\partial w_{e^\star}}(\theta) \geq b$ (the YES case) and $\frac{\partial L}{\partial w_{e^\star}}(\theta) \leq -b$ (the NO case). Moreover, $b \leq \poly(|I|)$ for a fixed polynomial.
\end{definition}

\bibliography{example_paper}
\bibliographystyle{plain}

\appendix

\section{Straight Line Programs (SLP)}
\label{sec:SLP}

In this section, we give the missing details on straight line programs (SLP). For the convenience of reading, we repeat some definitions and results given in \Cref{sec:prel}. We start with a more general definition of SLP.

\begin{definition}
    \label{def:SLP}
	A {\em straight-line program (SLP)} $P$ is a sequence of univariate integer polynomials $(a_0, a_1, \ldots , a_{\ell})$ such
	that $a_0 = 1$, $a_1 = x$ and $a_i = a_j \oplus a_k$~ for all $2 \leq i \leq \ell$, where $\oplus \in \{+, -, *\}$ and $j, k < i$. We use $\tau(f)$ to denote the minimum length of an SLP that computes a univariate polynomial $f$.
\end{definition}

In our hardness results, we use a conjecture proposed in \cite{10.1145/3510359}. A {\em radical} $\textnormal{\textsf{rad}}(f)$ of a non-zero integer polynomial $f \in \mathbb{Z}[x_1, \ldots , x_n]$ is the product of the irreducible
integer polynomials dividing $f$.

\begin{conjecture}
	\label{con:radical}
	\textnormal{\textsf{Constructive univariate radical conjecture:}} For any polynomial $f \in \mathbb{Z}[x]$, we have
	$\tau(\textnormal{\textsf{rad}}(f)) \leq \textnormal{\textsf{poly}}(\tau(f))$. Moreover, there is a randomized polynomial time algorithm which, given an \textnormal{SLP} of size
	$\ell$ computing a polynomial $f$, constructs an \textnormal{SLP} for $rad(f)$ of size $\textnormal{\textsf{poly}}(\ell)$ with success probability at least $1-\frac{1}{\Omega(\ell^{1+\varepsilon})}$
	for some
	$\varepsilon > 0$.
\end{conjecture}

In addition to polynomial SLPs as in Definition~\ref{def:SLP}, in this paper we use SLPs that compute
numbers (rather than univariate polynomials): a \emph{numeric} SLP is a sequence $(b_0,\ldots,b_m) \subseteq\mathbb{Q}$
such that $b_0$ is a fixed rational constant and for every $1\le i\le m$,
$b_i=b_j\oplus b_k$ for some $j,k<i$ and $\oplus\in\{+,-,\times\}$. For $b_0 = 1$, numeric SLP is equivalent to polynomial SLP evaluated at $x = 1$. Below, we show that we can generalize the initial constant to be a small dyadic rational, inducing a numeric SLP-like program  
in which the value of any gate in the SLP is bounded in $[-1,1]$. This will be used to prove our hardness results for bounded norms and show that the definition of SLPs given in \Cref{def:SLP_p} can indeed be set to a rational constant (rather than always $1$) without the loss of generality. 

\begin{definition}[Bounded-Norm SLP ($\mathsf{BN\text{-}SLP}$)]
	An instance is a numeric SLP $Q=(b_0,b_1,\ldots,b_m)$ of length $m$ over operations
	$\{+,-,\times\}$ whose \emph{constant} is
	$
	b_0 \;:=\;2^{-m}.
	$
\end{definition} Equivalently, this SLP is evaluated at $x = 2^{-m}$ instead of $x = 1$.
Thus, every gate value is a \emph{rational} number. Moreover, $b_0=1/2^m$ has a standard binary encoding of $O(m)$ bits. 

\begin{lemma}
	\label{clm:bnpslp-bounded}
	Let $Q=(b_0,b_1,\ldots,b_m)$ be a $\mathsf{BN\text{-}SLP}$ instance. Then, for every
	$i\in\{0,1,\ldots,m\}$ we have $b_i\in[-1,1]$.
\end{lemma}

\begin{proof}
	Let $M_i:=\max_{t\le i}|b_t|$. We prove by induction that $M_i\le 2^i b_0$ for all $i$,
	which implies $|b_i|\le M_m\le 2^m b_0=1$.
	The base case is $M_0=|b_0|=b_0\le 2^0 b_0$.
	For the induction step, fix $i\ge 1$. If $b_i=b_j\pm b_k$ then
	$|b_i|\le |b_j|+|b_k|\le 2M_{i-1}\le 2^i b_0$.
	If $b_i=b_j\cdot b_k$ then
	$|b_i|\le M_{i-1}^2\le (2^{i-1}b_0)^2=2^{2i-2}b_0^2\le 2^i b_0$,
	since $2^{i-2}b_0\le 2^{m-2}\cdot 2^{-m}=2^{-2}\le 1$.
	Thus $M_i\le 2^i b_0$ for all $i$.
\end{proof}

In the following, we show that an $\mathsf{SLP}$ can be effectively reduced to $\mathsf{BN\text{-}SLP}$ up to an exponential scale factor. This reduction cannot be done simply by scaling the first variable, since the addition/subtraction and multiplication combined would not scale correctly. 

\begin{lemma}[Reduction from $\mathsf{SLP}$ to $\mathsf{BN\text{-}SLP}$]
	\label{lem:posslp-to-bnpslp}
	There exists a deterministic polynomial-time reduction that, given an
	$\mathsf{SLP}$ $P$ of length $n$ computing an integer $n_P$ on $x = 1$, outputs a $\mathsf{BN\text{-}SLP}$
	instance $I(P)$ of length $m = \textnormal{poly}(n)$ that computes the rational $2^{-m \cdot 2^n} \cdot n_P$.
\end{lemma}

\begin{proof}
	Let $P=(a_0,a_1,\ldots,a_n)$ be an SLP with $a_0=1$ (i.e., $a_1(1)=a_0=1$) where each gate $a_i$ for $i\ge 1$
	is of the form $a_j\oplus a_k$ for some $0 \leq j,k<i$ and
	$\oplus\in\{+,-,\times\}$.
	We construct a $\mathsf{BN\text{-}SLP}$ instance by an SLP $Q$. We define the reduction in two passes. First, given $P$, we compute (in polynomial time)
	the exact number $m=m(P) = O\left(n^2\right)$ of gates that the construction below will output, by performing
	the same exponent bookkeeping and counting only the number of arithmetic gates added
	(ignoring the numeric value of $b_0$). Then we set $b_0:=2^{-m}$ and emit exactly $m$
	gates accordingly. Since $m(P)$ depends only on the syntactic structure of $P$ and the
	gate-counting procedure, this definition is non-circular and requires no padding.

	We construct an SLP $Q$ using the constant $b_0$ and the operations $\{+,-,\times\}$.
	For the construction, we use integers $e_i\ge 1$ (not used directly in the SLP) for $i\in\{0,1,\ldots,n\}$,
	such that for each original gate $a_i$ we compute a corresponding value $\tilde a_i$ satisfying
	the invariant
	\begin{equation}
		\label{eq:inv-bnpslp}
		\tilde a_i \;=\; b_0^{e_i}\, a_i.
	\end{equation}
	We set $\tilde a_0:=b_0$ and $e_0:=1$, so \eqref{eq:inv-bnpslp} holds for $i=0$ since $a_0=1$.
	We give an inductive construction of $\tilde a_i$ and $e_i$ for $i\ge 1$.

	\smallskip
	\noindent\emph{Multiplication.}
	If $a_i=a_j\cdot a_k$, define $\tilde a_i:=\tilde a_j\cdot \tilde a_k$ and
	$e_i:=e_j+e_k$. Then \eqref{eq:inv-bnpslp} holds for $i$.
	
	\smallskip
	\noindent\emph{Addition/Subtraction.}
	If $a_i=a_j\pm a_k$, let $E:=\max(e_j,e_k)$ and define
	\[
	\tilde a_j' := \tilde a_j\cdot b_0^{E-e_j},\qquad
	\tilde a_k' := \tilde a_k\cdot b_0^{E-e_k},
	\]
	(where $b_0^0:=1$ is implemented by leaving the value unchanged), and set
	$\tilde a_i:=\tilde a_j'\pm \tilde a_k'$ and $e_i:=E$. Then \eqref{eq:inv-bnpslp} holds for addition and subtraction as well. The exponents satisfy $e_0=1$ and, for $i\ge 1$, either $e_i\le \max(e_j,e_k)$
	(addition/subtraction) or $e_i=e_j+e_k$ (multiplication). Therefore $e_i\le 2^i$
	for all $i$ by induction. We define the output node $\tilde a_{n+1}$ of $Q$ receiving $b_0^{2^n - e_n} \cdot \tilde a_n$, giving $e_{n+1} = 2^n$ (if $e_n = 2^n$ we avoid adding this extra node). 
	
	\smallskip
	\noindent\emph{Implementing multiplication by $b_0^t$.}
	Fix an integer $t\ge 1$ and let $t=\sum_{i=0}^{\lfloor \log_2 t\rfloor} t_i 2^i$ with $t_i\in\{0,1\}$.
	We construct an SLP gadget $G_t$ over the sole constant $b_0$ and operation $\times$ such that,
	given an input gate value $z$, it outputs $z\cdot b_0^t$ and has size $O(\log t)$.
	Specifically, the gadget first computes $p_0:=b_0$ and
	$p_{r+1}:=p_r\cdot p_r=b_0^{2^{r+1}}$ for $r=0,\ldots,\lfloor \log_2 t\rfloor-1$.
	Then it sequentially multiplies $z$ by each $p_i$ with $t_i=1$, yielding $z\cdot b_0^t$.
	
	\smallskip
	\noindent\emph{Size bound.}
    Since $e_i\le 2^i$
	for all $i$, every alignment exponent
	$t\in\{E-e_j,E-e_k\}$ is at most $2^n$.
    Thus, each alignment multiplication by
	$b_0^t$ costs $O(\log t)=O(n)$ gates, and the total size is $m = O(n^2)$.
	
	\smallskip
	\noindent\emph{Correctness.}
	By~\eqref{eq:inv-bnpslp}, the simulated output satisfies
	$\tilde a_n = b_0^{e_n} a_n$ and therefore $\tilde a_{n+1} = b_0^{2^n} \cdot n_P$ as required. 
    Moreover, by Lemma~\ref{clm:bnpslp-bounded}, every gate value in $I(P)$ lies in $[-1,1]$.
\end{proof}

By Lemma~\ref{lem:posslp-to-bnpslp}, we can reduce BitSLP and PosSLP to rational bounded-norm variants, whose all intermediate values lie in $[-1,1]$. The hardness of BitSLP and PosSLP given in \Cref{thm:BitSLP} and \Cref{thm:PosSLP'}, respectively, is preserved for the BN-SLP variants since multiplying an integer by a small positive rational $2^{-m 2^n}$ preserves the sign and preserves all bit information shifted to the right by $m 2^n$ bits. Namely, it remains hard to decide the $j$-th bit of the SLP output, where $j$ one of the first $n$ bits to the right of the decimal point.
In addition, the positivity of an SLP output number is preserved by multiplying by a small factor; hence, \Cref{thm:PosSLP'} holds for BN-SLP instances as well. Therefore, we apply \Cref{thm:BitSLP} and \Cref{thm:PosSLP'} for regular SLP and BN-SLP instances interchangeably.

\section{Deferred Proofs from Section~\ref{sec:ERM}}
\label{sec:omited_3}
We prove below the main result of this section. 

\subsubsection*{Proof of Lemma~\ref{lem:SLP-to-ERM}}
\begin{proof}
	By Lemma~\ref{lem:polynomial_representation} applied to
	$f=\sigma\in\mathbb{Q}[T]$, there is a polynomial time algorithm that, given $\sigma$, returns $\lambda_0,\dots,\lambda_{\mu}\in\Q$
	such that for all $x,y\in\mathbb{Q}$ it holds that
	\begin{equation}
		\label{eq:mul-rep}
		\sum_{j=0}^{\mu} \lambda_j\bigl(
		\sigma(x+y+j) - \sigma(x+j) - \sigma(y+j)
		\bigr) = xy.
	\end{equation}
	Let $D$ be a common denominator of all the coefficients of $\sigma$ and
	all $\lambda_j$ (as rationals). Multiplying \eqref{eq:mul-rep} by $D$
	gives
	\begin{equation}
		\label{eq:mul-rep-int}
		D x y
		= \sum_{j=0}^{\mu} \lambda'_j\bigl(
		\sigma(x+y+j) - \sigma(x+j) - \sigma(y+j)
		\bigr),
	\end{equation}
	where $\lambda'_j := D\lambda_j\in\mathbb{Z}$.

	\paragraph{Construction of the network.}
	We now construct a directed acyclic network $N=(V,E)$, together with activations $\sigma^v$. We also construct a
	parameter vector $\theta^\ast$ that later will be shown to be the unique optimum. 
	
	\smallskip
	\noindent
	\textit{Vertices.}
	\begin{itemize}
		\item For each $i\in\{0,\dots,\ell\}$ create a vertex $v_i$. The source
		is $s:=v_0$ and the output is $t:=v_\ell$.
		\item For each gate $i$ in the SLP with $a_i = a_j \cdot a_k$ (a
		multiplication gate), and for each $r\in\{0,\dots,{\mu}\}$, create
		six auxiliary vertices:
		\[
		L^{(1)}_{i,r},\ L^{(2)}_{i,r},\ L^{(3)}_{i,r},\
		U^{(1)}_{i,r},\ U^{(2)}_{i,r},\ U^{(3)}_{i,r},
		\]
		and one additional vertex $z_i$.
	\end{itemize}
	
	\smallskip
	\noindent
	\textit{Activations.}
	\begin{itemize}
		\item For all $v\in\{v_0,\dots,v_\ell\}$ and for all $L^{(t)}_{i,r}$
		and $z_i$, set $\sigma^v(z) := z$ (identity).
		\item For all $U^{(t)}_{i,r}$, set $\sigma^v := \sigma$.
	\end{itemize}
	
	\smallskip
	\noindent
	\textit{Edges and the distinguished parameter vector $\theta^\ast$.}
	We define the edge set $E$ below. In addition, we define a \emph{distinguished}
	parameter vector $\theta^\ast=((w^\ast_{u,v},b^\ast_{u,v}))_{(u,v)\in E}\in\Theta$
	by specifying, for each edge $(u,v)\in E$, its target weight $w^\ast_{u,v}\in\Q$
	and bias $b^\ast_{u,v}\in\Q$ as follows.
	Note that in the ERM instance, all edge-parameters are trainable; $\theta^\ast$ is
	used only as the target solution in the analysis and for defining auxiliary samples.
	
	\begin{itemize}
		\item \emph{Source.} For $s=v_0$ we take $f^s(x)=x$ by definition (no
		incoming edges).
		
		\item \emph{Addition / subtraction gates.}
		For each gate $i$ with $a_i = a_j + a_k$ and $j,k<i$:
		\begin{itemize}
			\item Add edges $e_{j\to i}=(v_j,v_i)$ and $e_{k\to i}=(v_k,v_i)$.
			\item Set $w^\ast_{e_{j\to i}}=1$, $w^\ast_{e_{k\to i}}=1$.
			\item Set $b^\ast_{e_{j\to i}}=b^\ast_{e_{k\to i}}=0$.
		\end{itemize}
		For each gate $i$ with $a_i = a_j - a_k$:
		\begin{itemize}
			\item Add edges $e_{j\to i}=(v_j,v_i)$ and $e_{k\to i}=(v_k,v_i)$.
			\item Set $w^\ast_{e_{j\to i}}=1$, $w^\ast_{e_{k\to i}}=-1$. 
			\item Set $b^\ast_{e_{j\to i}}=b^\ast_{e_{k\to i}}=0$.
		\end{itemize}
		
		\item \emph{Multiplication gates.}
		For each gate $i$ with $a_i = a_j \cdot a_k$ and $j,k<i$, and for
		each $r\in\{0,\dots,{\mu}\}$:
		\begin{itemize}
			\item Add edges from $v_j$ and $v_k$ to $L^{(1)}_{i,r}$, both
			with weight $w^\ast_{v_j \to L^{(1)}_{i,r}} = 1$, $w^\ast_{v_k \to L^{(1)}_{i,r}} = 1$, and add a bias $r$ (encode this
			as a bias on one of the incoming edges, chosen
			arbitrarily, such that $b^\ast_{v_j \to L^{(1)}_{i,r}} + b^\ast_{v_k \to L^{(1)}_{i,r}} = r$). Thus
			$f^{L^{(1)}_{i,r}}(a_0)=f^{v_j}(a_0)+f^{v_k}(a_0)+r$.
			\item Add an edge from $v_j$ to $L^{(2)}_{i,r}$ with weight $w^\ast_{v_j \to L^{(2)}_{i,r}}=1$
			and bias $b^\ast_{v_j \to L^{(2)}_{i,r}} = r$, so $f^{L^{(2)}_{i,r}}(a_0)=f^{v_j}(a_0)+r$.
			\item Add an edge from $v_k$ to $L^{(3)}_{i,r}$ with weight $w^\ast_{v_k \to L^{(3)}_{i,r}} = 1$
			and bias $b^\ast_{v_k \to L^{(3)}_{i,r}} = r$, so $f^{L^{(3)}_{i,r}}(a_0)=f^{v_k}(a_0)+r$.
			\item Add edges $\left(L^{(t)}_{i,r},U^{(t)}_{i,r}\right)$ with
			weight $w^\ast_{L^{(t)}_{i,r} \to U^{(t)}_{i,r}} = 1$ and bias $b^\ast_{L^{(t)}_{i,r} \to U^{(t)}_{i,r}} =0$ for $t\in\{1,2,3\}$.
			\item For all $t \in \{1,2,3\}$ Add an edge from $U^{(t)}_{i,r}$
			to $z_i$ with weight 
            \[
w^\ast_{U^{(t)}_{i,r} \to z_i} =
\begin{cases}
\lambda'_r & \text{if } t=1,\\
-\lambda'_r & \text{if } t=2,3.
\end{cases}
\]
            
			 with zero bias $b^\ast_{U^{(t)}_{i,r} \to z_i} = 0$.
		\end{itemize}
		Finally, add an edge $(z_i,v_i)$ with weight $w^\ast_{z_i \to v_i} = D^{-1}$ and bias
		$b^\ast_{z_i \to v_i} = 0$ for each multiplication gate $i$ (i.e., $a_i = a_j \cdot a_k$).
	\end{itemize}
	
	Since $\sigma$ and the coefficients $\lambda_j$ are fixed, $D$ and all
	$\lambda'_r\in\mathbb{Z}$ are fixed rational constants with bounded bit-length,
	so the weights and biases in $\theta^\ast$ have
	bit-length $O_\sigma(1)$ depending only on $\sigma$ (and not on the SLP
	instance encoding size). Thus, the encoding size of the resulting instance
	$N$ and $\theta^\ast$ is polynomial in the length of $P$.
	
	\paragraph{Correctness of the SLP simulation under $\theta^\ast$.}
	In the ERM model below the input is a vector $x\in\Q^{V}$ that injects additively at every vertex.
	Whenever we write $f^v_{\theta^\ast}(a_0)$ in this simulation part, we mean
	$f^v_{\theta^\ast}(x)$ for the special input $x$ with $x_{v_0}=a_0$ and $x_v=0$ for all $v\neq v_0$ (later defined in the ERM instance as $x_{\textnormal{main}}$).
	
	We prove by induction on $i$ that it holds that
	\[
	f^{v_i}(a_0) = a_i \quad\text{for all } i\in\{0,\dots,\ell\}.
	\]
	\emph{Base case:} $i=0$. We have $f^{v_0}(a_0)=f^s(a_0)=a_0$.
	
	\emph{Induction step:} Assume $f^{v_j}(a_0)=a_j$ for all $j<i$. Consider a
	gate $i$ with $a_i = a_j \oplus a_k$ and $j,k<i$.
	
	If $\oplus=+$, then by construction and induction hypothesis,
	\[
	f^{v_i}(a_0)
	= f^{v_j}(a_0) + f^{v_k}(a_0)
	= a_j + a_k = a_i.
	\]
	If $\oplus=-$, the same argument with weights $1$ and $-1$ gives
	$f^{v_i}(a_0)=a_j-a_k=a_i$.
	
	If $\oplus=*$, then by the induction hypothesis
	$f^{v_j}(a_0)=a_j$ and $f^{v_k}(a_0)=a_k$. Let $x:=f^{v_j}(a_0)$ and
	$y:=f^{v_k}(a_0)$, so $x=a_j$ and $y=a_k$. For each $r \in \{0,\ldots,{\mu}\}$ we have
	\[
	f^{L^{(1)}_{i,r}}(a_0) = x+y+r,
	f^{L^{(2)}_{i,r}}(a_0) = x+r,
	f^{L^{(3)}_{i,r}}(a_0) = y+r.
	\]
	In addition by construction,
	\[
	f^{U^{(t)}_{i,r}}(a_0)
	= \sigma\!\bigl(f^{L^{(t)}_{i,r}}(a_0)\bigr)
	\quad\text{for }t\in\{1,2,3\}.
	\]
	Then, by the definition of $z_i$,
	\begin{align*}
		f^{z_i}(a_0)
		&= \sum_{r=0}^{\mu} \lambda'_r\Bigl(
		f^{U^{(1)}_{i,r}}(a_0)
		- f^{U^{(2)}_{i,r}}(a_0)
		- f^{U^{(3)}_{i,r}}(a_0)
		\Bigr) \\
		&= \sum_{r=0}^{\mu} \lambda'_r\Bigl(
		\sigma(x+y+r) - \sigma(x+r) - \sigma(y+r)
		\Bigr).
	\end{align*}
	Applying \eqref{eq:mul-rep-int} yields
	$f^{z_i}(a_0) = D a_j a_k$. Finally, the edge $(z_i,v_i)$ has weight
	$D^{-1}$ and identity activation, so
	\[
	f^{v_i}(a_0) = D^{-1} f^{z_i}(a_0) = a_j \cdot a_k = a_i.
	\]
	
	This completes the induction. In particular,
	\begin{equation}
		\label{eq:scaled-output-new}
		f_{\theta^\ast}(a_0) = f^{t}(a_0) = a_\ell = n_P.
	\end{equation}

	\paragraph{Defining the ERM instance.}
	In the constructed ERM instance, \emph{every vertex} is both an input coordinate and an output coordinate.
	Thus, the input space is $\cX=\Q^{V}$ and the predictor outputs the full vector
	$f_\theta(x)\in\Q^{V}$, where for a topological order of $V$ the node values are computed, for every $v \in V$ by
	\begin{equation}
		\label{eq:node-rule}
		f^v_\theta(x)
		\;:=\;
		\sigma^v\!\Bigl(x_v + \sum\nolimits_{(u,v)\in E}\bigl(w_{u,v}\,f^u_\theta(x)+b_{u,v}\bigr)\Bigr).
	\end{equation}
	
	\smallskip
	Fix a promise gap $a<b$ (with $a,b\in\N$). 
	Let $j$ be the queried bit index from the BitSLP instance.
	We define the label space as $\cY:=\Q^{V}\times\{0,1\}$, where the second component is a flag indicating whether the sample is auxiliary (i.e., $0$) or main (i.e., $1$).
	
	\smallskip
	We use the following efficiently computable bit extractor on rationals:
	write $q=u/v$ with $u\in\Z$, $v\in\N$ in lowest terms and define
	\[
	\mathrm{bit}_j(q):=\Big\lfloor 2^{-j}\tfrac{|u|}{v}\Big\rfloor \bmod 2.
	\]
	Observe that this is computable in time polynomial in the bit-length of $(u,v,j)$.
	
	\smallskip
	Define the loss $\cL:\Q^{V}\times(\Q^{V}\times\{0,1\})\to\{0,1\}$ by\footnote{Without the loss of generality, we avoid producing a flag for the output of the network; by adding an arbitrary flag to the output, we can define the loss as $\cL:\cY \times \cY \to\{0,1\}$}
	\[
	\cL(\hat y,(y,\mathrm{flag})) :=
	\begin{cases}
		\mathbf{1}\bigl[\mathrm{bit}_j(\hat y_t)\neq 1\bigr], & \text{if }\mathrm{flag}=1,\\[1mm]
		\mathbf{1}\bigl[\hat y\neq y\bigr], & \text{if }\mathrm{flag}=0.
	\end{cases}
	\]
	In the above, recall that $t$ is the target node of the network. Note that equality of rational vectors is polynomial-time decidable by coordinatewise rational equality.

\paragraph{Auxiliary samples.}
We use auxiliary samples that force the trainable parameter vector $\theta$ to implement the distinguished vector $\theta^\ast$
(up to $\sigma$-symmetries on $\sigma$-edges).
Recall that $\mu=\deg(\sigma)\ge 2$.
Since $\sigma$ is a nonconstant polynomial, the polynomial $\sigma(z)-\sigma(0)$ has at most $\mu$ roots,
so there exists $\alpha_1\in\{1,2,\dots,\mu+1\}$ with $\sigma(\alpha_1)\neq\sigma(0)$; fix such an $\alpha_1$.
Set $\Delta_\sigma:=\sigma(\alpha_1)-\sigma(0)\neq 0$.

\smallskip
Define the \emph{baseline output template} $y^{(0)}\in\Q^{V}$ by
\[
y^{(0)}_v :=
\begin{cases}
	0, & \text{if }\sigma^v=\mathrm{id},\\
	\sigma(0), & \text{if }\sigma^v=\sigma
\end{cases}
\]
for every $v \in V$. For auxiliary constructions we will specify, for each template, not only the desired node-values
$y\in\Q^{V}$ but also a \emph{preactivation template} $\tilde y\in\Q^{V}$ such that
\[
y_v=
\begin{cases}
	\tilde y_v, & \text{if }\sigma^v=\mathrm{id},\\
	\sigma(\tilde y_v), & \text{if }\sigma^v=\sigma.
\end{cases}
\]
Given such a pair $(y,\tilde y)$, define an input $x(y,\tilde y)\in\Q^{V}$ which for every $v\in V$ is
\begin{equation}
	\label{eq:aux-input}
	x(y,\tilde y)_v \;:=\; \tilde y_v - \sum\nolimits_{(u,v)\in E}\bigl(w^\ast_{u,v}\,y_u + b^\ast_{u,v}\bigr).
\end{equation}
By a direct induction along any topological order of $V$, the choice \eqref{eq:aux-input} enforces
\begin{equation}
	\label{eq:aux-realize}
	f_{\theta^\ast}\bigl(x(y,\tilde y)\bigr)=y.
\end{equation}
Moreover, in all auxiliary templates below, every coordinate of $y$ and $\tilde y$
is an explicit rational obtained from the constants
$$\{0,1,\dots,\mu\}\cup\{\sigma(0),\sigma(\alpha_1)\}\cup\sigma(\{0,1,\dots,\mu\})$$
together with coefficients appearing in $\theta^\ast$ (namely in $\{D^{-1},\pm 1,\pm \lambda'_r\}$)
using only $O_\sigma(1)$ many additions and multiplications (since every indegree is $O_\sigma(1)$).
Hence each coordinate of $x(y,\tilde y)$ defined in \eqref{eq:aux-input} has at most polynomial bit-length,
and the total encoding size of each auxiliary example is polynomial in $|P|$.

\smallskip
Let $E_{\mathrm{id}}:=\{(u,v)\in E:\sigma^v=\mathrm{id}\}$ and let
\[
\begin{aligned}
	E_{\sigma}
	:=& \Bigl\{\bigl(L^{(t)}_{i,r},U^{(t)}_{i,r}\bigr):\  i \text{ is a multiplication gate},\ r\in\{0,\dots,\mu\},
	& t\in\{1,2,3\}\Bigr\}.
\end{aligned}
\]
The latter set contains exactly the edges whose head is a $\sigma$-node.

\smallskip
Define the baseline preactivation template $\tilde y^{(0)}\in\Q^{V}$ by $\tilde y^{(0)}_v:=0$ for all $v\in V$
(so $y^{(0)}_v=\sigma^v\left(\tilde y^{(0)}_v\right)$ for every $v$).

\smallskip
For each edge $e=(p,q)\in E_{\mathrm{id}}$, define a template pair $(y^{(e)},\tilde y^{(e)})$ as follows:
\[
\begin{aligned}
	y^{(e)}_v &:= y^{(0)}_v \qquad &&\text{for all } v\notin\{p,q\},\\
	y^{(e)}_p &:= 
	\begin{cases}
		1, & \text{if }\sigma^p=\mathrm{id},\\
		\sigma(\alpha_1), & \text{if }\sigma^p=\sigma,
	\end{cases}
	\qquad &&\\
	y^{(e)}_q &:= w^\ast_{p,q}\bigl(y^{(e)}_p - y^{(0)}_p\bigr), \qquad &&
\end{aligned}
\]
and set the preactivation template by
\[
\tilde y^{(e)}_v :=
\begin{cases}
	y^{(e)}_v, & \text{if }\sigma^v=\mathrm{id},\\
	0, & \text{if }\sigma^v=\sigma \text{ and } v\neq p,\\
	\alpha_1, & \text{if }\sigma^p=\sigma \text{ and } v=p.
\end{cases}
\]
Note that $y^{(e)}_p-y^{(0)}_p$ is either $1$ or $\Delta_\sigma$, hence nonzero.

For each $\sigma$-edge $e_\sigma=(L,U)\in E_\sigma$ and each $\tau\in\{0,1,\dots,\mu\}$, define a template pair $(y^{(e_\sigma,\tau)},\tilde y^{(e_\sigma,\tau)})$ by:
\[
\begin{aligned}
	y^{(e_\sigma,\tau)}_v &:= y^{(0)}_v &&\text{for all } v\notin\{L,U\},\\
	y^{(e_\sigma,\tau)}_L &:= \tau,\\
	y^{(e_\sigma,\tau)}_U &:= \sigma(\tau),
\end{aligned}
\]
\[
\begin{aligned}
	\tilde y^{(e_\sigma,\tau)}_v &:= 0 &&\text{for all } v\notin\{L,U\},\\
	\tilde y^{(e_\sigma,\tau)}_L &:= \tau,\\
	\tilde y^{(e_\sigma,\tau)}_U &:= \tau.
\end{aligned}
\]

(Then $x(y^{(e_\sigma,\tau)},\tilde y^{(e_\sigma,\tau)})_U=0$ by \eqref{eq:aux-input}, since $U$ has the single incoming edge $(L,U)$ with
$w^\ast_{L,U}=1$ and $b^\ast_{L,U}=0$.)

\smallskip
Define the auxiliary multiset
\[
\begin{aligned}
	\cD_a &\;:=\;
	\Bigl\{(x(y^{(0)},\tilde y^{(0)}),(y^{(0)},0))\Bigr\}^{\times(b+1)}
	\\
	&\cup\;
	\bigcup_{e\in E_{\mathrm{id}}}
	\Bigl\{(x(y^{(e)},\tilde y^{(e)}),(y^{(e)},0))\Bigr\}^{\times(b+1)}
	\\
	&\cup\;
	\bigcup_{e_\sigma\in E_{\sigma}}\ \bigcup_{\tau=0}^{\mu}
	\Bigl\{(x(y^{(e_\sigma,\tau)},\tilde y^{(e_\sigma,\tau)}),(y^{(e_\sigma,\tau)},0))\Bigr\}^{\times(b+1)}.
\end{aligned}
\]
By \eqref{eq:aux-realize}, $\theta^\ast$ incurs zero loss on $\cD_a$.

	\paragraph{Main samples.}
	Define the main input $x_{\mathrm{main}}\in\Q^{V}$ by
	\[
	x_{\mathrm{main}}(v_0)=a_0,\qquad x_{\mathrm{main}}(v)=0\ \ (v\neq v_0).
	\]
	Let $y_{\mathrm{main}}\in\Q^{V}$ be arbitrary (e.g.\ $0$ everywhere), and set the main label to $(y_{\mathrm{main}},1)$.
	Let $\cD_m$ contain $b$ identical copies of $(x_{\mathrm{main}},(y_{\mathrm{main}},1))$.
	
	\smallskip
	Finally, define the overall dataset as $\cD:=\cD_a\cup\cD_m$. We have the following claim used for proving the auxiliary samples force zero loss solutions to satisfy $\theta$ has similar attributes to  $\theta^*$.

\begin{claim}\label{clm:force-all-trainable}
	If $\theta$ satisfies 
	\[\sum_{(x,(y,\mathrm{flag}))\in\cD_a}\cL(f_\theta(x),(y,\mathrm{flag}))=0,\]
	then:
	\begin{itemize}
		\item For every identity-head edge $(p,q)\in E_{\mathrm{id}}$, one has $w_{p,q}=w^\ast_{p,q}$ and
		$\sum_{(u,q)\in E} b_{u,q}=\sum_{(u,q)\in E} b^\ast_{u,q}$.
		\item For every $\sigma$-edge $(L,U)\in E_\sigma$, the edge parameters satisfy the polynomial identity
		$
		\sigma\bigl(w_{L,U}z+b_{L,U}\bigr)\equiv \sigma(z)
		$
		as polynomials in $z$; thus, on every input $x$ one has $f^U_\theta(x)=\sigma\!\bigl(f^L_\theta(x)\bigr)$.
	\end{itemize}
\end{claim}
\begin{proof}
	Zero auxiliary loss implies that for every auxiliary template used above,
	\[
	f_\theta\bigl(x(y,\tilde y)\bigr)=y.
	\]
	
	\smallskip
	\noindent\emph{Identity-head edges.}
	Fix $e=(p,q)\in E_{\mathrm{id}}$ and recall the auxiliary samples $(y^{(0)},\tilde y^{(0)})$ and $(y^{(e)},\tilde y^{(e)})$.
	Since $\sigma^q=\mathrm{id}$, by \eqref{eq:node-rule} and the equalities
	$f_\theta(x(y^{(0)},\tilde y^{(0)}))=y^{(0)}$ and $f_\theta(x(y^{(e)},\tilde y^{(e)}))=y^{(e)}$, we have
	\begin{align*}
		0
		&= y^{(0)}_q
		= x(y^{(0)},\tilde y^{(0)})_q + \sum_{(u,q)\in E}\bigl(w_{u,q}\,y^{(0)}_u+b_{u,q}\bigr),\\
		y^{(e)}_q
		&= x(y^{(e)},\tilde y^{(e)})_q + \sum_{(u,q)\in E}\bigl(w_{u,q}\,y^{(e)}_u+b_{u,q}\bigr).
	\end{align*}
	By construction, $y^{(e)}_u=y^{(0)}_u$ for all $u\neq p,q$ and
	$x(y^{(e)},\tilde y^{(e)})_q=x(y^{(0)},\tilde y^{(0)})_q$
	(this holds because $y^{(e)}_q=w^\ast_{p,q}(y^{(e)}_p-y^{(0)}_p)$ was chosen exactly to keep the $q$-input unchanged).
	Subtracting the two displayed equations yields
	\[
	y^{(e)}_q \;=\; w_{p,q}\bigl(y^{(e)}_p-y^{(0)}_p\bigr).
	\]
	But by definition $y^{(e)}_q=w^\ast_{p,q}(y^{(e)}_p-y^{(0)}_p)$ and $y^{(e)}_p-y^{(0)}_p\neq 0$ (it is either $1$ or $\Delta_\sigma$), so
	$w_{p,q}=w^\ast_{p,q}$.
	
	With $w_{u,q}=w^\ast_{u,q}$ for all $(u,q)\in E$ established by ranging over all $p$,
	plugging into the baseline equation for $q$ and using the definition of $x(y^{(0)},\tilde y^{(0)})_q$ in \eqref{eq:aux-input} gives
	\[
	\sum_{(u,q)\in E} b_{u,q}
	=
	\sum_{(u,q)\in E} b^\ast_{u,q},
	\]
	as claimed.
	
	\smallskip
	\noindent\emph{$\sigma$-edges.}
	Fix $(L,U)\in E_\sigma$. For every $\tau\in\{0,1,\dots,\mu\}$, the auxiliary sample corresponding to $(y^{(e_\sigma,\tau)},\tilde y^{(e_\sigma,\tau)})$ enforces
	$f^L_\theta(x)=\tau$ and $f^U_\theta(x)=\sigma(\tau)$, and by construction $x(y^{(e_\sigma,\tau)},\tilde y^{(e_\sigma,\tau)})_U=0$
	and $U$ has no incoming edges other than $(L,U)$.
	Thus, by \eqref{eq:node-rule} and the condition of the claim, 
	\[
	\sigma\!\bigl(w_{L,U}\tau+b_{L,U}\bigr)=\sigma(\tau)\qquad\text{for all }\tau\in\{0,1,\dots,\mu\}.
	\]
	Define $p(z):=\sigma(w_{L,U}z+b_{L,U})-\sigma(z)$, a univariate polynomial of degree at most $\mu$.
	The above shows that $p$ has $\mu+1$ distinct roots, hence $p\equiv 0$, i.e.,
	$\sigma(w_{L,U}z+b_{L,U})\equiv\sigma(z)$ as polynomials.
	Consequently, for every input $x$,
	\[
	f^U_\theta(x)=\sigma\!\bigl(w_{L,U}f^L_\theta(x)+b_{L,U}\bigr)=\sigma\!\bigl(f^L_\theta(x)\bigr).
	\]
	This completes the proof of the claim.
\end{proof}

	Assume $\theta$ has zero auxiliary loss. By Claim~\ref{clm:force-all-trainable},
	all identity-head edge weights match $\theta^\ast$, and every $\sigma$-node $U^{(t)}_{i,r}$ satisfies
	$f^{U^{(t)}_{i,r}}_\theta(x)=\sigma(f^{L^{(t)}_{i,r}}_\theta(x))$ for all $x$.
	Therefore, the same induction as in the SLP simulation part applies to $x_{\mathrm{main}}$:
	addition/subtraction gates are computed with the correct weights w.r.t. $\theta^*$,
	and at each multiplication gate the gadget evaluates exactly the right expression in \eqref{eq:mul-rep-int},
	so $f^t_\theta(x_{\mathrm{main}})=n_P$.
	
	\paragraph{Completing the reduction.}
	Consider the empirical loss over $\cD$.
	If $\theta$ has nonzero loss on any auxiliary sample, then since every auxiliary sample is repeated $(b+1)$ times,
	\[
	\sum_{(x,(y,\mathrm{flag}))\in\cD}\cL(f_\theta(x),(y,\mathrm{flag})) \;\ge\; b+1 \;>\; b
	\]
	so any $\theta$ with loss at most $b$ must have zero auxiliary loss.
	For such $\theta$, we have $f^t_\theta(x_{\mathrm{main}})=n_P$ as above, and each main sample contributes loss
	$0$ iff $\mathrm{bit}_j(n_P)=1$, and loss $1$ otherwise.
	Since there are $b$ identical main samples, the main-part loss is either $0$ or $b$.
	Because $a<b$, there exists $\theta$ with total loss $\le a$ iff $\mathrm{bit}_j(n_P)=1$.
	
	The construction is polynomial-time in $|P|$:
	$|V|,|E|=O(|P|)$, all weights/biases in $\theta^\ast$ have $O_\sigma(1)$ bit-length,
	and every auxiliary input $x(y,\tilde y)$ is computed by \eqref{eq:aux-input} from constant-sized templates.
	This yields a polynomial-time reduction from BitSLP to $\mathsf{ERM}_{\text{bit}}$. 
\end{proof}

\begin{remark}[Reducing both input and output dimensions to $1$ via centered $L_2$ regularization]
\label{rem:one-dim-via-l2}
The above reduction can also be implemented with scalar input and scalar output
(\(\cX=\cY=\Q\)) by using a centered quadratic regularizer
\(\lambda\|\theta-\theta^\ast\|_2^2\) in the objective.
The idea is to keep the same circuit network and use a single-bit loss on the
output value (querying the \(j\)-th bit), replicated \(b\) times.

At \(\theta^\ast\), the network output is \(n_P\), so the data term is \(0\) iff
the \(j\)-th bit of \(n_P\) is \(1\), and \(b\) otherwise. Then choose
\(\lambda\) (as a dyadic rational \(2^M\) with polynomial-size encoding) large
enough so that every feasible \(\theta\neq\theta^\ast\) contributes at least
\(b\) through \(\lambda\|\theta-\theta^\ast\|_2^2\). Intuitively, this follows
from the bit-model encoding bound: distinct feasible rational parameters are
separated by an inverse-exponential gap, and \(\lambda\) is picked to dominate
that gap uniformly.

Hence any optimum is attained at \(\theta^\ast\), and the YES/NO decision again
matches whether the \(j\)-th bit of \(n_P\) is \(1\). This gives the same
hardness with one-dimensional input and output.
\end{remark}

As mentioned in the discussion, we explain in the following how to use a standard hinge loss instead of the bit-extraction loss to obtain a slightly weaker intractability result for $\mathsf{ERM}_{\text{bit}}$

\begin{remark}
\label{rem:hinge-posslp}
Consider the (standard) hinge loss
\[
\mathcal{L}_{\mathrm{hinge}}(u,y)\;:=\;\max\{0,\,1-yu\},
\qquad y\in\{+1,-1\}.
\]
Assume we have the same training/architecture gadget as in the main reduction above,
with a distinguished parameter vector $\theta^*$ and input $a_0$ such that
the network output satisfies $f_{\theta^*}(a_0)=n_Q$ for an integer produced by an SLP $Q$,
and with the same auxiliary forcing mechanism used in the main reduction to ensure
that any optimizer equals $\theta^*$. We describe the following reduction from PosSLP to $\mathsf{ERM}_{\text{bit}}$ with hinge loss.
Given an SLP $P$ computing an integer $n_P$, construct an SLP $Q$ computing
\[
n_Q \;:=\; 2n_P-1.
\]
Then $n_Q\neq 0$ always (since $n_P\in\mathbb{Z}$), and
\[
n_P>0 \iff n_Q\ge 1,
\text{while }
n_P\le 0 \iff n_Q\le -1.
\]
Create an $\mathsf{ERM}_{\text{bit}}$ instance with $b$ identical labeled samples
$(x,y)=(a_0,+1)$, together with the same auxiliary forcing samples used in the
main reduction so that the optimizer is pinned to $\theta^*$.
Then
the optimum value equals $\mathcal{L}_{\mathrm{hinge}}(n_Q,+1)=\max\{0,1-n_Q\}$, hence:
\[
n_Q\ge 1 \ \Rightarrow\ \mathrm{OPT}=0
\]
and in addition
\[
n_Q\le -1 \ \Rightarrow\ \mathrm{OPT}\ge 2.
\]
Therefore, deciding whether $\mathrm{OPT}\le a$ or $\mathrm{OPT}\ge b$ is \textsf{PosSLP}-hard with a more standard activation. 
\end{remark}

We now give the proof Lemma~\ref{lem:polynomial_representation}

\subsubsection*{Proof of \Cref{lem:polynomial_representation}}
\begin{proof}
	
	Let $P_{\mu} := \mathbb{Q}[U]_{\le \mu}$ be the $(\mu+1)$‑dimensional $\mathbb{Q}$‑vector
	space of polynomials in one variable $U$ of degree at most $\mu$.
	For any $j=0,1,\dots,\mu$, define $f_j(U) := f(U+j)$ and note that $f_j \in P_{\mu}$. We say that $B$ is a $\mathbb{Q}$‑basis of $P_{\mu}$ if it is a basis of $P_{\mu}$ with coefficients taken only from~$\mathbb{Q}$. 
	
	\begin{claim}\label{lem:basis}
		The family $\{f_0,\dots,f_{\mu}\}$ is a $\mathbb{Q}$‑basis of $P_{\mu}$.
	\end{claim}
	
	\begin{proof}
		Since $\deg f_j = {\mu}$ for all $j$ and the dimension of the vector space $P_{\mu}$ over the rationals is $\dim_{\mathbb{Q}} P_{\mu} = {\mu}+1$,
		it suffices to prove that $f_0,\dots,f_{\mu}$ are linearly independent
		over $\mathbb{Q}$.
		Suppose
		$
		\sum_{j=0}^{\mu} \alpha_j f_j(U) \equiv 0
		$ for some $\alpha_j \in \mathbb{Q}$ for all $j = 0,\ldots, {\mu}$.
		We show $\alpha_j = 0$ for all $j = 0,\ldots, {\mu}$.
Recall that 
		$
		f_j(U)
		= f(U+j)
		= \sum_{k=0}^{\mu} r_k (U+j)^k
		$,
		thus, 
		\begin{align*}
		0
		&= \sum_{j=0}^{\mu} \alpha_j f_j(U)\\
		&=  \sum_{j=0}^{\mu}  \alpha_j  \sum_{k=0}^{\mu} r_k \cdot (U+j)^k\\
		&= \sum_{k=0}^{\mu} r_k  \sum_{j=0}^{\mu} \alpha_j (U+j)^k.
		\end{align*}

		This is the zero polynomial in $U$; therefore, all its coefficients must be zero. We consider these coefficients from highest degree to lowest. We use this to prove the following claim by induction:
		We show below by induction that \[
		\sum_{j=0}^{\mu} \alpha_j \cdot j^\ell = 0
		\quad\text{for all } \ell = 0,1,\dots,{\mu}.
		\]

		\emph{Coefficient of $U^{\mu}$.}
		Only the $k={\mu}$ term contributes to the coefficient of $U^{\mu}$ in the above. In addition, for each $j = 0,\ldots,{\mu}$ the coefficient of $U^{\mu}$ in $\sum_{k=0}^{\mu} r_k  \sum_{j=0}^{\mu} \alpha_j (U+j)^k$ is
		$r_{\mu} \cdot \sum_{j=0}^{\mu} \alpha_j$.
		Since $r_{\mu} \neq 0$, we obtain
	$
		\sum_{j=0}^{\mu} \alpha_j = 0
		$. 
		The above gives the base cases of the induction. 
		
		\smallskip
		\emph{Inductive step.}
		Assume that for some $m$ with $1 \le m \le \mu$ it holds that
		\[
		\sum_{j=0}^{\mu} \alpha_j \cdot j^\ell = 0
		\quad\text{for all } \ell = 0,1,\dots,m-1.
		\]
		Consider the coefficient of $U^{{\mu}-m}$ in $\sum_j \alpha_j f_j(U)$.
		By the binomial formula, for every $k \in \mathbb{N}$ it holds that 
		\[
		(U+j)^k
		= \sum_{q=0}^k \binom{k}{q} j^{k-q} U^q.
		\]
		Thus, the coefficient of $U^{{\mu}-m}$ in $\sum_{j=0}^{\mu} \alpha_j f_j(U)$ is
		\[
		\sum_{k=\mu-m}^{\mu}
		r_k \binom{k}{{\mu}-m}
		\sum_{j=0}^{\mu} \alpha_j \cdot j^{k-({\mu}-m)}.
		\]
		For $k<\mu$, we have $k-({\mu}-m) \le m-1$, so by the inductive hypothesis
		\(\sum_j \alpha_j \cdot j^{k-({\mu}-m)} = 0\). Hence, all contributions with $k<{\mu}$
		are zero, and only $k={\mu}$ remains. Thus, the coefficient is:
		\[
		r_{\mu} \cdot \binom{{\mu}}{{\mu}-m} \cdot \sum_{j=0}^{\mu} \alpha_j \cdot j^m = 0.
		\]
		Since $r_{\mu} \neq 0$ and $\binom{{\mu}}{{\mu}-m} \neq 0$, we obtain
		\[
		\sum_{j=0}^{\mu} \alpha_j \cdot j^m = 0.
		\]
		
		By induction on $m=0,1,\dots,{\mu}$, we conclude that
		\[
		\sum_{j=0}^{\mu} \alpha_j j^\ell = 0
		\quad\text{for all } \ell = 0,1,\dots,{\mu}.
		\]
		Let $V$ be the Vandermonde matrix $V = (v_{\ell j})$ with
		$v_{\ell j} = j^\ell$ for $\ell,j=0,\dots,{\mu}$.
		The above conditions say exactly that $V \alpha = 0$, where
		$\alpha = (\alpha_0,\dots,\alpha_{\mu})^T$. By the known determinant of Vandermonde matrix formula, it holds that 
		\[
		\det V = \prod_{0 \le i < j \le {\mu}} (j - i).
		\]
		Since $0,1,\dots,{\mu}$ are distinct, $\det V \neq 0$. Therefore, $V$ is invertible
		over $\mathbb{Q}$, implying that $\alpha = 0$.
		Therefore, $f_0,\ldots, f_{\mu}$ are linearly independent. 
		Thus, $\{f_0,\dots,f_{\mu}\}$ is a $\mathbb{Q}$‑basis of $P_{\mu}$ as required. 
			\end{proof}

	Since $U^2 \in P_{\mu}$ and $\{f_0,\dots,f_{\mu}\}$ is a basis of $P_{\mu}$, there exist
	unique $\lambda_0,\dots,\lambda_{\mu} \in \mathbb{Q}$ such that
	\begin{equation}\label{eq:U2-expansion}
		U^2 = \sum_{j=0}^{\mu} \lambda_j f(U+j).
	\end{equation}
	The coefficients $\lambda_j$ indeed lie in $\mathbb{Q}$ because the change‑of‑basis matrix from
	$\{1,U,\dots,U^{\mu}\}$ to $\{f_0,\dots,f_{\mu}\}$ has entries in $\mathbb{Q}$ and
	is invertible over $\mathbb{Q}$, so its inverse also has entries in
	$\mathbb{Q}$.
	
	Identity \eqref{eq:U2-expansion} holds in $\mathbb{Q}[U]$, so in
	$\mathbb{Q}[X,Y]$ we may substitute:
	\begin{align}
		X^2       &= \sum_{j=0}^{\mu} \lambda_j f(X+j), \label{eq:X2} \\
		Y^2       &= \sum_{j=0}^{\mu} \lambda_j f(Y+j), \label{eq:Y2} \\
		(X+Y)^2   &= \sum_{j=0}^{\mu} \lambda_j f(X+Y+j). \label{eq:XY2}
	\end{align}
	
	In $\mathbb{Q}[X,Y]$ we have
	$
	X \cdot Y = \frac{(X+Y)^2 - X^2 - Y^2}{2}.
	$
	Using~\eqref{eq:X2}–\eqref{eq:XY2},
	\begin{equation}
		XY = \sum_{j=0}^{\mu} \frac{\lambda_j}{2}\bigl(f(X+Y+j) - f(X+j) - f(Y+j)\bigr).
	\end{equation}
Define \[g(X,Y) = \sum_{j=0}^{\mu} \frac{\lambda_j}{2}\bigl(f(X+Y+j) - f(X+j) - f(Y+j)\bigr)\]
The above implies that
 for all $x,y \in \mathbb{R}$ it follows that $g(x,y) = x \cdot y$ with $\lambda'_j = \lambda_j / 2$.

 \paragraph{Computing the coefficients $\lambda_j$ in polynomial time.}
Assume $f(U)=\sum_{k=0}^{\mu} r_k U^k\in\Q[U]$ is given by its coefficients
$r_k\in\Q$ in binary encoding. For each $j\in\{0,\ldots,{\mu}\}$ we can expand
\[
f_j(U)=f(U+j)=\sum_{k=0}^{\mu} r_k (U+j)^k=\sum_{q=0}^{\mu} c_{q,j} U^q
\]
using the binomial theorem, where each $c_{q,j}\in\Q$ is computable in time
$\poly({\mu},\mathrm{bits}(f))$, denoting by $\mathrm{bits}(f)$ the number of bits needed to encode $f$. Let $C\in\Q^{({\mu}+1)\times({\mu}+1)}$ be the matrix whose
$j$-th column is the coefficient vector $(c_{0,j},\ldots,c_{{\mu},j})^{\top}$ of
$f_j$ in the basis $(1,U,\ldots,U^{\mu})$. Then~\eqref{eq:U2-expansion} is
equivalent to the linear system
\[
C\,\lambda = e_2,
\]
where $\lambda=(\lambda_0,\ldots,\lambda_{\mu})^{\top}$ and $e_2$ is the coefficient
vector of $U^2$ (i.e., $1$ in position $2$ and $0$ elsewhere). By
Claim~\ref{lem:basis}, $C$ is invertible over $\Q$, hence the system has a unique
solution in $\Q^{{\mu}+1}$. Solving a $(\mu{+}1)\times({\mu}{+}1)$ linear system over
$\Q$ (e.g. by Gaussian elimination with exact rational arithmetic) takes time
$\poly({\mu},\mathrm{bits}(f))$, and produces $\lambda_0,\ldots,\lambda_{\mu}$ with
polynomially bounded bit-length.
\end{proof}

We can now prove our main result of this section. 

\begin{proof} [Proof of Theorem~\ref{thm:SLP}]
 By Lemma~\ref{lem:SLP-to-ERM}, BitSLP Turing-reduces to $\mathsf{ERM}_{\text{bit}}$. 
Thus, by~\Cref{thm:BitSLP} the proof follows.
\end{proof}

\section{Deferred Proofs from Section~\ref{sec:backprop}}

We now prove our hardness for Backprop-Sign and Backprop-Bit described in \Cref{thm:backprop-hard-bit}, holding even for an arbitrarily large promise gap (for PosSLP) or for all intermediate node values uniformly bounded by a constant (for BN-SLP). For brevity, we write one reduction for both of the above cases and distinguish between the cases when required. 

\subsubsection*{Proof of \Cref{thm:backprop-hard-bit}}
\begin{proof}
	We give polynomial-time reductions from \textnormal{PosSLP} to
	\textnormal{Backprop-Sign} and from \textnormal{BitSLP} to
	\textnormal{Backprop-Bit}. In the former, we define a \textnormal{Backprop-Sign} instance with a promise gap $a<b$ for arbitrary fixed $a,b \in \mathbb{N}$. For the latter case, we simply define a \textnormal{Backprop-Bit} instance. We give both constructions together for conciseness and highlight the differences when required. 
	
	Let $P=(a_0,\dots,a_\ell)$ be an SLP. Applying
	Lemma~\ref{lem:SLP-to-ERM} to $P$, we obtain a network
	$N_0=(V_0,E_0)$, activations $\sigma^v\in\{\sigma,\textnormal{id}\}$, and rational
	parameters $\theta^0=(w_e^0,b_e^0)_{e\in E_0}$ such that on input
	$x=a_0$,
	\[
	f_{\theta^0}(a_0) = n_P := a_\ell\in\mathbb{Q}.
	\]
	Let $t_0\in V_0$ denote the output vertex of $N_0$.
%
	We form a new network $N=(V,E)$ as follows:
	\begin{itemize}
		\item Add a new target vertex $t$ with identity activation $\sigma^t(z)=z$.
		\item Add a new edge $e^\star := (t_0,t)$.
		\item Set $V := V_0 \cup \{t\}$ and $E := E_0 \cup \{e^\star\}$.
	\end{itemize}
	We fix all parameters on $E_0$ to their values in $\theta^0$, and set
	the bias on $e^\star$ to $0$. The only free parameter will be the weight
	$w_{e^\star}$ on $e^\star$.
	For a parameter vector
	$
	\theta = \bigl((w_e,b_e)_{e\in E_0}, (w_{e^\star},b_{e^\star})\bigr),
	$
	with $(w_e,b_e)=(w_e^0,b_e^0)$ and $b_{e^\star}=0$, the output on input
	$x=a_0$ is
	\[
	f^t_\theta(a_0)
	= w_{e^\star} \cdot f^{t_0}_{\theta^0}(a_0)
	= w_{e^\star} \cdot n_P.
	\]

For defining the two reductions seamlessly, let $B$ be a variable defined as $B= 1$ in the reduction from BitSLP and $B = b$ in the reduction from PosSLP.
We consider $B$ identical training samples
\[
\cD = \{(x_i,y_i)\}_{i \in [B]},\qquad x_i=a_0,\quad y_i=-a_0.
\]
We consider the standard square loss: $\cL:\mathbb{Q}\times\mathbb{Q}\to\mathbb{Q}$ defined by
\[
\cL(x,y) := \tfrac{1}{2}(x-y)^2.
\]
This loss is polynomial-time computable and differentiable in its first
argument, and its partial derivative is
$\partial \cL/\partial x(x,y) = x-y$, which is polynomial-time computable
and rational-valued on rational inputs. The total training loss is then
\[
\begin{aligned}
	L(\theta)
	&= \sum_{(x,y)\in\cD} \cL\bigl(f_\theta(x),y\bigr) \\
	&= \tfrac{B}{2}\bigl(f^t_\theta(a_0)-y\bigr)^2 \\
	&= \tfrac{B}{2}\bigl(w_{e^\star} n_P + 1\bigr)^2.
\end{aligned}
\]
We view $L$ as a real-valued function of the real variables
$(w_e,b_e)_{e\in E}$, defined by the same algebraic expression. Then for
all such $\theta$ with the fixed sub network parameters and bias
$b_{e^\star}=0$, we have
\[
\frac{\partial L}{\partial w_{e^\star}}(\theta)
= B \bigl(f^t_\theta(a_0)-y\bigr)\,\frac{\partial f^t_\theta(a_0)}{\partial w_{e^\star}}
= B \bigl(w_{e^\star} n_P + a_0\bigr)\, n_P.
\]

In particular, for the specific parameter vector
\[
\theta^\ast := \bigl((w_e^0,b_e^0)_{e\in E_0},
(w_{e^\star}=0,b_{e^\star}=0)\bigr),
\]
since $w_{e^\star}=0$ we have
\[
\frac{\partial L}{\partial w_{e^\star}}(\theta^\ast)
= B \cdot a_0 \cdot n_P.
\]
	
	 We now consider two cases. If $P$ is a PosSLP instance, define the Backprop-Sign instance with promise corresponding to $P$ as
	\[
	I(P) := (N,\cD,\cL,\theta^\ast,e^\star,b).
	\]
    Otherwise, $P$ is a BitSLP instance and we define the Backprop-Bit instance corresponding to $P$ as
	\[
	I(P) := (N,\cD,\cL,\theta^\ast,e^\star,j),
	\]
    where $j \in \mathbb{N}$ is given in the input in binary.
	All rationals used in $N$ and $\theta^\ast$ (the parameters from
	$\theta^0$ and the fixed bias $0$) have bit-length polynomial in the
	size of $P$ by Lemma~\ref{lem:SLP-to-ERM}; the loss $\cL$ is given by a
	constant-size program, and $\cD = \{(x_i,y_i)\}_{i \in [b]}$ also has polynomial bit-length encoding for fixed $b$.
	Hence, the total encoding size of $I(P)$ is polynomial in the encoding
	size of $P$, and the reduction $P\mapsto I(P)$ runs in polynomial time.
	By the above,
	\[
	\frac{\partial L}{\partial w_{e^\star}}(\theta^\ast) = B \cdot a_0 \cdot n_P > 0
	\quad\Longleftrightarrow\quad n_P>0.
	\]
	Therefore, $P$ is a YES-instance of PosSLP if and only if $I(P)$ is a
	YES-instance of Backprop-Sign. This gives a polynomial-time
	reduction from PosSLP to Backprop-Sign. Note that if $P$ is a PosSLP instance, the gap $b\geq$ or $\leq -b$ is preserved if $a_0 = 1$. 
    
    In addition, considering a BitSLP instance $(j,P)$, then $B = 1$ and we can directly decide the $j$-th bit in $n_P$ by directly accessing it; if $a_0 = 1$, this is immediate; otherwise, we consider a BN-SLP with $a_0 = 2^{-\ell}$, which shifts the $j$-th bit of $n_P$ to the right by exactly $\ell$ bits. Therefore, we can recover the $j$-th bit of $n_P$ in polynomial time. 

    By the above and \Cref{thm:BitSLP} it follows that Backprop-Bit is $\#P$-hard. In addition, 
	if Backprop-Sign were in BPP, then PosSLP would be in BPP. By
	Theorem~\ref{thm:PosSLP'}, under the Constructive Univariate Radical Conjecture \cite{10.1145/3510359} together with PosSLP $\in$ BPP,
	it follows that $\textnormal{NP}\subseteq\textnormal{BPP}$, contradicting our
 assumption. Thus, Backprop-Sign is not in BPP.
\end{proof}

\section{PAC Learning in Real vs. Bit Models}
\label{sec:PAC}

In this section, we prove the following theorem.

\begin{theorem}
	\label{thm:pac-separation}
	For every integer $q\ge 1$ there is a hypothesis class and distribution
	for which the realizable \textnormal{PAC} sample complexity is $O(1)$ in the exact
	\textnormal{real-RAM} model, but for every rounding precision $q$ in the bit model it is $\Omega(q)$.  
\end{theorem}

We begin with the formal definition of (realizable) PAC learning. 
\begin{definition}[Realizable PAC learning]
	Let $\cX$ be a domain, $\cY$ a label space, and $\cH\subseteq\cY^\cX$
	a hypothesis class. Let $D$ be a distribution on $\cX$. For
	$h\in\cH$ and a hypothesis $\hat h:\cX\to\cY$, the (0-1) error of
	$\hat h$ with respect to $h$ and $D$ is
	\[
	\mathrm{err}_D(\hat h,h)
	:= \Pr_{x\sim D}\bigl[\hat h(x) \ne h(x)\bigr].
	\]
	A (possibly randomized) learning algorithm $\cA$ is said to
	\emph{realizably PAC learn} $(\cH,D)$ with parameters
	$(\varepsilon,\delta)$ if there exists a sample size $m$ such that,
	for every target $h^\star\in\cH$, whenever $\cA$ is given $m$
	i.i.d.\ labeled examples $(x_i,h^\star(x_i))$ with $x_i\sim D$, it
	outputs a hypothesis $\hat h$ satisfying
	\[
	\Pr\bigl[\mathrm{err}_D(\hat h,h^\star) \le \varepsilon\bigr]
	\;\ge\; 1-\delta,
	\]
	where the probability is over the draw of the sample and the
	internal randomness of $\cA$.
\end{definition}

We now define a simple one-parameter family of rounded multipliers at
precision $2^{-q}$ and show that, even in the realizable setting under
a fixed distribution, the number of samples needed to drive the error
below $1/(q+1)$ grows linearly with~$q$.

\begin{definition}[Rounded multipliers at precision $2^{-q}$]
	Fix an integer $q\ge 1$ and define the domain
	\[
	\cX_q := \{2^0,2^1,\dots,2^{q}\} \subset \mathbb{R}.
	\]
	For $z\in\mathbb{R}$ let
	\[
	\mathrm{round}_q(z)
	:= 2^{-q} \bigl\lfloor z\,2^{q} \bigr\rfloor,
	\]
	i.e., rounding $z$ down to the nearest multiple of $2^{-q}$. For each
	real parameter $c\in[0,1)$ define
	\[
	h_c(x) := \mathrm{round}_q(c x),\quad x\in\cX_q.
	\]
	Let $\cH_q := \{h_c : c\in[0,1)\}$ and let $D_q$ be the uniform
	distribution on $\cX_q$. We use 0-1 loss on the discrete labels
	$\cY := \{2^{-q} k : k\in\mathbb{Z}\}$, so that
	$\mathrm{err}_{D_q}(\hat h,h)$ is just the fraction of points in
	$\cX_q$ where $\hat h$ and $h$ differ.
\end{definition}

In the \emph{exact real} model (no rounding) and noiseless setting, a
single labeled example $(x,y)$ with $x\neq 0$ and $y=cx$ determines
$c=y/x$ and thus $h_c$ on all of $\cX_q$. The next theorem shows that
under rounding at precision $2^{-q}$, even in the realizable PAC model
under the fixed distribution $D_q$, the sample complexity to achieve
error $<1/(q+1)$ grows linearly with~$q$.

\begin{theorem}
	\label{thm:pac-rounding-lb}
	Fix $q\ge 1$ and consider the class $\cH_q$ and distribution $D_q$
	defined above. Let $\varepsilon < 1/(q+1)$ and let
	$\delta \in (0,1/4)$. Then any (possibly randomized) algorithm that
	realizably \textnormal{PAC} learns $(\cH_q,D_q)$ with parameters $(\varepsilon,\delta)$
	must use  sample size
	$
	m \;\ge\; q\,\log\frac{1}{2\delta}
	$
\end{theorem}

\begin{proof}
	Fix $q\ge 1$. We first prove the lower bound for deterministic
	algorithms, and then extend to randomized algorithms by a standard
	argument.
	
	We will exhibit two functions $h^{(1)},h^{(2)}\in\cH_q$ such that:
	\begin{itemize}
		\item $h^{(1)}(x) = h^{(2)}(x)$ for all $x\in\{2^0,\dots,2^{q-1}\}$,
		\item $h^{(1)}(2^q) \ne h^{(2)}(2^q)$.
	\end{itemize}
	As a consequence, any sample that does not contain the point $2^q$
	looks identical under $h^{(1)}$ and $h^{(2)}$.
	
	Write any $c\in[0,1)$ in binary as
	$c = 0.b_1 b_2 b_3 \dots$ with $b_i\in\{0,1\}$ choosing one expansion arbitrarily when two exist. For $k\ge 0$,
	\[
	c 2^k
	= \sum_{i=1}^k b_i 2^{k-i} \;+\; \theta_k,
	\]
	with $\theta_k\in[0,1)$, hence
	$\lfloor c 2^k\rfloor = \sum_{i=1}^k b_i 2^{k-i}$. Applying this with
	$k+q$ in place of $k$,
	\[
	h_c(2^k)
	= \mathrm{round}_q(c 2^k)
	= 2^{-q}\bigl\lfloor c 2^{k+q}\bigr\rfloor
	= 2^{-q} \sum_{i=1}^{k+q} b_i 2^{k+q-i}.
	\]
	Thus for $k\in\{0,\dots,q-1\}$, $h_c(2^k)$ depends only on the first
	$k+q \le 2q-1$ bits of $c$.
	
	Now fix any binary sequence $b_1,\dots,b_{2q-1}$. Define
	$c^{(1)},c^{(2)}\in[0,1)$ by
	\[
	c^{(1)} = 0.b_1 b_2 \dots b_{2q-1} 0 0 0 \dots,
		\]
		and 
			\[
	c^{(2)} = 0.b_1 b_2 \dots b_{2q-1} 1 0 0 \dots,
	\]
	and let $h^{(j)} := h_{c^{(j)}}$ for $j\in\{1,2\}$. For any
	$k\in\{0,\dots,q-1\}$ we have $k+q \le 2q-1$, so the first $k+q$ bits
	of $c^{(1)}$ and $c^{(2)}$ coincide, hence
	\[
	h^{(1)}(2^k) = h^{(2)}(2^k) \quad\text{for all }k<q.
	\]
	On the other hand, at $x^\star:=2^q$,
	\[
	h^{(j)}(2^q)
	= 2^{-q} \bigl\lfloor c^{(j)} 2^{2q} \bigr\rfloor.
	\]
	Since $c^{(1)}$ and $c^{(2)}$ agree on their first $2q-1$ bits but
	differ at bit $2q$, we have
	\[
	\bigl\lfloor c^{(2)} 2^{2q} \bigr\rfloor
	- \bigl\lfloor c^{(1)} 2^{2q} \bigr\rfloor
	= 1,
	\]
	so $h^{(1)}(2^q)\ne h^{(2)}(2^q)$.
	
	Fix a deterministic learning algorithm $\cA$ and a sample size $m$. We
	will show that if $m < q \cdot \log(1/(2\delta))$, then there exists a
	target $h^\star\in\cH_q$ (either $h^{(1)}$ or $h^{(2)}$) such that
	$\cA$ cannot achieve error at most $\varepsilon<1/(q+1)$ with
	probability at least $1-\delta$.
	
	Consider the following experiment: pick $h^\star$ uniformly at random
	from $\{h^{(1)},h^{(2)}\}$ and draw an i.i.d.\ sample
	$S=((x_i,y_i))_{i=1}^m$ with $x_i\sim D_q$ and
	$y_i = h^\star(x_i)$. Let $\hat h_S$ be the output of $\cA$ on sample
	$S$.
	
	Let $E$ be the event that the sample $S$ contains no occurrence of
	$x=2^q$. Since $D_q$ is uniform on $\{2^0,\dots,2^q\}$,
	\[
	\Pr[E] = \Bigl(1 - \tfrac{1}{q+1}\Bigr)^m
	= \Bigl(\tfrac{q}{q+1}\Bigr)^m.
	\]
	On the event $E$, all sampled points lie in $\{2^0,\dots,2^{q-1}\}$,
	and we have seen that $h^{(1)}$ and $h^{(2)}$ coincide on these inputs.
	Hence, conditioned on $E$, the labeled sample $S$ has exactly the same
	distribution regardless of whether $h^\star=h^{(1)}$ or
	$h^\star=h^{(2)}$. Therefore, on $E$, the output $\hat h_S$ of the
	deterministic algorithm $\cA$ is the same in both cases.
	
	However, we have $h^{(1)}(2^q)\ne h^{(2)}(2^q)$, so on event $E$, that is, for every sample $S$ that does not contain $2^q$,
	\emph{at least one} of the two possible targets leads to an error at
	$2^q$:
\[
\mathbf{1}\{\hat h_S(2^q)\ne h^{(1)}(2^q)\}
\;+\;
\mathbf{1}\{\hat h_S(2^q)\ne h^{(2)}(2^q)\}
\;\ge\; 1.
\]
	Taking expectation over the random choice of $h^\star$ (uniform on
	$\{h^{(1)},h^{(2)}\}$) and of $S$, we obtain
	\begin{align*}
		\Pr_{h^\star,S}\bigl[\hat h_S(2^q)\ne h^\star(2^q)\bigr]
		&\ge \tfrac{1}{2}\Pr[E] \\
		&= \tfrac{1}{2}\Bigl(\tfrac{q}{q+1}\Bigr)^m.
	\end{align*}

Since $D_q$ is uniform, note that for any fixed $h^\star\in\cH_q$ and hypothesis $\hat h$,
\[
\mathrm{err}_{D_q}(\hat h,h^\star)
= \frac{1}{q+1}\sum_{x\in\cX_q}
\mathbf{1}\{\hat h(x)\ne h^\star(x)\},
\]
so $\mathrm{err}_{D_q}(\hat h,h^\star)\in\{0,1/(q+1),\dots,1\}$. In
particular, if $\hat h_S(2^q)\ne h^\star(2^q)$ then
$\mathrm{err}_{D_q}(\hat h_S,h^\star)\ge 1/(q+1)>\varepsilon$, and hence
\[
\Pr_{h^\star,S}\!\bigl[
\mathrm{err}_{D_q}(\hat h_S,h^\star)>\varepsilon
\bigr]
\;\ge\; \tfrac12\Bigl(\tfrac{q}{q+1}\Bigr)^{m}.
\]

By averaging over the random choice of
$h^\star\in\{h^{(1)},h^{(2)}\}$, there exists a fixed
$h^\star\in\cH_q$ such that
\[
\Pr_S\!\bigl[
\mathrm{err}_{D_q}(\hat h_S,h^\star)>\varepsilon
\bigr]
\;\ge\; \tfrac12\Bigl(\tfrac{q}{q+1}\Bigr)^{m}.
\]
Requiring this probability to be at most $\delta$ and using $\log(1+1/q)\le 1/q$ as $q > 0$ yields
\[
\tfrac12\Bigl(\tfrac{q}{q+1}\Bigr)^{m} \le \delta
\;\Rightarrow\;
m \;\ge\; q\,\log\!\Bigl(\tfrac{1}{2\delta}\Bigr).
\]

Any randomized learner is a distribution over deterministic learners
(obtained by fixing its internal randomness). If a randomized learner
with $m$ samples satisfied the $(\varepsilon,\delta)$ guarantee for all
$h^\star\in\cH_q$, then by averaging some deterministic learner in its
support would also satisfy it, contradicting the lower bound above.
Thus, the same lower bound on $m$ holds for randomized learners as well.
\end{proof}

\section{Efficient Backpropagation and ERM for Piecewise-Linear Activations}
\label{sec:efficient}

In contrast to our hardness results for polynomial activations, standard
piecewise-linear nonlinearities such as ReLU and leaky-ReLU do not introduce
new multiplicative interactions in the bit model.  In this section we show
that, for such activations, both a single backpropagation step and
verification of ERM objectives can be carried out in polynomial time, placing
these problems in $\mathsf{P}$ (and hence $\mathsf{ERM}_{\text{bit}}$ in $\mathsf{NP}$).

\paragraph{Bit-length conventions.}
Let $N$ denote the total bit-length of the input encoding (architecture,
parameters $\theta$, dataset $(x_i,y_i)_{i=1}^n$, and learning rate $\eta$).
In particular, $n \le N$, $|E| \le N$, and every scalar appearing in the input
(weights, biases, inputs, labels, $\eta$) has bit-length at most $N$.
Recall that all arithmetic is over rationals in binary; if $a,b$ have bit-lengths at most
$B,C$, then
\[
\mathrm{len}(a\pm b) \le \max\{B,C\}+1,
\qquad
\mathrm{len}(ab) \le B+C+1.
\]

\begin{theorem}[One backpropagation step with piecewise-linear activations is in $\mathsf{P}$]
	Assume all hidden and output neurons use piecewise-linear activations
	$\sigma^v:\mathbb{Q}\to\mathbb{Q}$ with rational breakpoints and rational
	slopes and intercept, and that
	each $\sigma^v$ and its derivative with respect to the input are
	computable in time polynomial in the bit-length of their arguments.  Assume
	also that $\cL$ and its derivative with respect to the network output are
	computable in time polynomial in the bit-length of their arguments.
	
	Then one full gradient descent step for
	\[
	J(\theta)=\sum_{i=1}^n \cL\bigl(f_\theta(x_i),y_i\bigr)
	\]
	computed by standard backpropagation runs in polynomial time in $|I|$.
\end{theorem}

\begin{proof}
Note that the argument below uses only that each activation $\sigma^v$ (and its backprop rule/derivative)
is computable in polynomial time on rational inputs and maps inputs of bit-length $\poly(N)$ to outputs of
bit-length $\poly(N)$; therefore the same proof applies to any \emph{bit-bounded} (e.g., $k$-bit rounded/clipped)
activation family with $k\le \poly(N)$.

	\emph{Operation count.}
	Standard backpropagation (forward pass, backward pass, gradient accumulation,
	and parameter update) uses $O(n\cdot|E|)$ additions, multiplications, and
	comparisons; since $n,|E|\le N$, this is polynomially bounded in $N$.
	
	\emph{Bit-length growth for activations and errors.}
	Fix an example $(x_i,y_i)$. All preactivations $z_v^{(i)}$, activations
	$h_v^{(i)}$, and error terms $\delta_v^{(i)}$ are obtained from the input
	scalars (coordinates of $x_i,y_i$ and parameters $\theta$) by repeated
	application of:
	\begin{itemize}
		\item addition / subtraction of previously computed values,
		\item multiplication of a previously computed value by a \emph{parameter}
		(weight or bias) or by a fixed rational slope or coefficient from some
		piece of a $\sigma^v$,
		\item evaluation of a piecewise-linear $\sigma^v$ or its derivative
		$\partial\sigma^v/\partial z$, which consists of a constant number of
		comparisons with rational breakpoints and then an affine expression of the
		form $a z + b$ with $a,b\in\mathbb{Q}$.
	\end{itemize}
	All slopes, intercepts, and breakpoints of all $\sigma^v$ are part of the
	input and therefore have bit-length at most $N$. Thus, starting from values
	of bit-length at most $N$, along the computation of any $z_v^{(i)}$,
	$h_v^{(i)}$, or $\delta_v^{(i)}$ we perform at most $O(|E|)$ such arithmetic
	steps, each of which increases bit-length by at most $O(N)$ (multiplication
	by a rational of bit-length $\le N$) or $1$ (addition), while comparisons and
	case distinctions do not change bit-length. Hence every $z_v^{(i)}$,
	$h_v^{(i)}$, and $\delta_v^{(i)}$ has bit-length bounded by
	\[
	N + O(|E|)\cdot(N+1) \;\le\; \mathrm{poly}(N).
	\]
	
	\emph{Bit-length of gradients and update.}
	For each edge $u\to v$,
	\[
	\frac{\partial J}{\partial w_{u\to v}}
	= \sum_{i=1}^n \delta_v^{(i)} h_u^{(i)},
	\qquad
	\frac{\partial J}{\partial b_v}
	= \sum_{i=1}^n \delta_v^{(i)}.
	\]
	Each product $\delta_v^{(i)} h_u^{(i)}$ involves factors of bit-length
	$\mathrm{poly}(N)$ and therefore has bit-length $\mathrm{poly}(N)$; summing
	over $i=1,\dots,n$ adds at most $\log n \le \log N$ bits. Thus every gradient
	entry has bit-length $\mathrm{poly}(N)$.
	
	The parameter update
	$\theta \leftarrow \theta - \eta\,\nabla J(\theta)$ consists of multiplications
	of gradient entries by the input scalar $\eta$ (bit-length $\le N$) and
	additions/subtractions with the previous parameters (bit-length $\le N$),
	so the updated parameters also have bit-length bounded by $\mathrm{poly}(N)$.
	
	\emph{Conclusion.}
	During one backpropagation step, all intermediate values have bit-length
	$\mathrm{poly}(N)$. Each arithmetic operation on such numbers can be carried
	out in time polynomial in $N$ in our bit model. Since the number of operations
	is $O(n\cdot|E|)\le O(N^2)$, the total running time is polynomial in $N$.
	Hence one backpropagation step with piecewise-linear activations lies in
	$\mathsf{P}$.
\end{proof}

We can now also prove Remark~\ref{cor:bit-bounded}
\paragraph{Proof of Remark~\ref{cor:bit-bounded}}
Identical to the proof above: the operation count is unchanged, and condition (ii) directly enforces that all
intermediate quantities have bit-length $\poly(N)$, so each arithmetic operation costs $\poly(N)$ time.

\begin{theorem}
	Assuming piecewise-linear activations,
	the \textnormal{restricted $\mathsf{ERM}_{\text{bit}}$} problem in the bit model is in $\mathsf{P}$. In particular,
	\textnormal{ERM} in the bit model (with explicit rational witnesses) lies in
	$\mathsf{NP}$.
\end{theorem}

\begin{proof}
	In a restricted $\mathsf{ERM}_{\text{bit}}$ instance we are given a fixed parameter vector
	$\theta^\ast$ as part of the input. To decide the instance we must compute
	\[
	\sum_{i=1}^n \cL\bigl(f_{\theta^\ast}(x_i),y_i\bigr)
	\]
	and compare it to $\gamma$.
	
	This is exactly the forward part of the backpropagation computation (without
	the backward and update steps). As in the bit-length analysis above, each
	activation $h_v^{(i)}$ is obtained from inputs and parameters by a sequence of
	additions, multiplications by parameters or fixed slopes, and
	piecewise-linear activations of length at most $O(|E|)$, so all intermediate
	values (and hence each $f_{\theta^\ast}(x_i)$ and the sum of losses) have
	bit-length bounded by $\mathrm{poly}(N)$. The number of operations is
	$O(n\cdot|E|)\le O(N^2)$, so this computation runs in polynomial time in $N$.
	Therefore, restricted $\mathsf{ERM}_{\text{bit}}$ with piecewise-linear activations in the bit model is in $\mathsf{P}$.
	
	Since, in the general ERM (bit-model) problem, a candidate parameter vector
	$\theta$ is given explicitly as a rational witness, and verifying its loss is
	in $\mathsf{P}$ by the same argument, ERM in the bit model lies in
	$\mathsf{NP}$.
\end{proof}

We note that NP-Completeness follows from the (numerous) known NP-Hardness results given in the literature (e.g., \cite{judd1988complexity,blum1988training,vsima2002training}).

\section{Not Every Rational Polynomial Can Represent $x \cdot y$}
\label{sec:lower_bound_poly}
In this section, we study when a fixed univariate function $f$ can be used as a
``universal gadget'' to represent the product of two variables using only
shifted evaluations of $f$ and rational coefficients.  Concretely, we ask for
identities of the form
\[
x y \;=\; \sum_{j} \lambda_j\bigl(f(x+y+j) - f(x+j) - f(y+j)\bigr),
\]
valid for all reals $x,y$, where the $\lambda_j$ are rational and only
finitely many are nonzero.

We have shown in \Cref{lem:polynomial_representation} that every nonlinear polynomial with rational coefficients
admits such a representation. 
One might hope to extend this positive result from integer powers to
non-integer rational exponents or negative integer exponents.  The next theorem shows that this is
impossible: for $f(T) = T^\alpha$ with $\alpha \in \mathbb{Q}
\setminus \mathbb{Z}_{\ge 0}$, no finite rational combination of shifted copies of
$f$ can represent $x \cdot y$.

\begin{theorem}
	\label{thm:counter_fractional}
	Let $\alpha \in \mathbb{Q}\setminus\mathbb{Z}_{\ge 0}$ and consider
	$f(T) = T^\alpha$ on $(0,\infty)$, with the principal real branch of
	$t^\alpha$. Then there do not exist an integer $m \ge 0$, rational
	coefficients $\lambda_0,\dots,\lambda_m \in \mathbb{Q}$, and a real
	number $R>0$ such that for all rational $x,y>R$,
	\begin{equation}
		\label{eq:frac_representation}
		\sum_{j=0}^m \lambda_j\bigl(f(x+y+j) - f(x+j) - f(y+j)\bigr) = x  \cdot y.
	\end{equation}
\end{theorem} In particular, the representation of
Lemma~\ref{lem:polynomial_representation} for polynomials cannot be
extended to $f(T)=T^\alpha$ with rational $\alpha \notin \mathbb{Z}_{\ge 0}$.
To prove this result, we use the following. 

\begin{lemma}
	\label{lem:indep}
	Let $\beta \in \mathbb{Q} \setminus \mathbb{Z}_{\ge 0}$ and let
	$b_0,\dots,b_m \in \mathbb{Q}$ be pairwise distinct. Define
	\[
	\varphi_j(x) := (x+b_j)^\beta
	\]
	on an interval $(R,\infty)$ with $R>-\min_j b_j$, so that $x+b_j>0$ for all
	$x>R$.
%
	If
	\[
	\sum_{j=0}^m \lambda_j \varphi_j(x) = 0
	\quad\text{for all } x>R,
	\]
	for some $\lambda_0,\dots,\lambda_m \in \mathbb{Q}$, then
	$\lambda_0=\cdots=\lambda_m=0$.
\end{lemma}

\begin{proof}
	Assume
	\[
	\sum_{j=0}^m \lambda_j (x+b_j)^\beta = \sum_{j=0}^m \lambda_j \varphi_j(x) =0
	\quad\text{for all } x>R.
	\]
	Set $t = 1/x$ for $x>R$; then for all sufficiently small $t>0$,
	it holds that $x+b_j = 1/t + b_j > 0$; thus,
	\[
	0 = \sum_{j=0}^m \lambda_j (1/t + b_j)^\beta.
	\]
	Multiplying by $t^\beta$ gives
	\[
	0
	= t^\beta \sum_{j=0}^m \lambda_j (1/t + b_j)^\beta
	= \sum_{j=0}^m \lambda_j (1 + b_j t)^\beta.
	\]
	Define
	\[
	F(t) := \sum_{j=0}^m \lambda_j (1 + b_j t)^\beta.
	\]
	For $|t|$ small, each $(1 + b_j t)^\beta$ admits a convergent binomial
	series, hence $F$ is analytic near $t=0$. We have $F(t)=0$ for all
	sufficiently small $t>0$, so by analyticity $F(t)\equiv 0$ near $0$.
	
	Using the binomial expansion,
	\[
	(1 + b_j t)^\beta
	= \sum_{k=0}^\infty \binom{\beta}{k} b_j^k t^k,
	\]
	we obtain
	\begin{align*}
		0 = F(t)
		&= \sum_{j=0}^m \lambda_j \sum_{k=0}^\infty
		\binom{\beta}{k} b_j^k t^k \\
		&= \sum_{k=0}^\infty
		\binom{\beta}{k}\Bigl(\sum_{j=0}^m \lambda_j b_j^k\Bigr) t^k.
	\end{align*}
Thus, since $F$ is analytic near $0$ and $F(t)\equiv 0$, for every $k\ge 0$:
\[
\binom{\beta}{k} \sum_{j=0}^m \lambda_j b_j^k = 0.
\]
Since $\beta\notin\mathbb{Z}_{\ge 0}$, the binomial coefficient
\[
\binom{\beta}{k}
= \frac{\beta(\beta-1)\cdots(\beta-k+1)}{k!}
\]
is nonzero for all $k\ge 0$, so
	\[
	\sum_{j=0}^m \lambda_j b_j^k = 0
	\quad\text{for all } k\ge 0.
	\]
	
	For $k=0,1,\dots,m$, this yields the homogeneous linear system
	$V\lambda = 0$, where $V$ is the $(m+1)\times(m+1)$ Vandermonde matrix
	$V_{kj} = b_j^k$. Since the $b_j$ are pairwise distinct, $\det V\neq 0$,
	so $V$ is invertible and hence $\lambda_0=\cdots=\lambda_m=0$.
\end{proof}

We now give the proof of \Cref{thm:counter_fractional}
\begin{proof}
	
	Assume towards a contradiction that such $m$, $\lambda_0,\dots,\lambda_m$, and
	$R$ exist. For $x,y>R$ all arguments of $f$ in
	\eqref{eq:frac_representation} are positive, so $f$ is real-analytic
	there. Both sides of \eqref{eq:frac_representation} are analytic in
	$(x,y)$ on $(R,\infty)^2$ and agree on all rational pairs $(x,y)$ in this
	region, hence by analyticity the identity holds for all real $x,y>R$.
	
	Differentiating \eqref{eq:frac_representation} with respect to $x$ gives
	\begin{equation}
		\label{eq:first_derivative}
		y
		= \sum_{j=0}^m \lambda_j \alpha\Bigl(
		(x+y+j)^{\alpha-1} - (x+j)^{\alpha-1}
		\Bigr)
	\end{equation}
	for all $x,y>R$. Differentiating \eqref{eq:first_derivative} with respect
	to $y$ yields
	\begin{equation}
		\label{eq:second_derivative}
		1
		= \sum_{j=0}^m \lambda_j \alpha(\alpha-1) (x+y+j)^{\alpha-2}
	\end{equation}
	for all $x,y>R$.
	
	Let $z = x+y$. Then for all $z>2R$, \eqref{eq:second_derivative} becomes
	\begin{equation}
		\label{eq:constant_eq_frac}
		1
		= \alpha(\alpha-1)\sum_{j=0}^m \lambda_j (z+j)^{\alpha-2}.
	\end{equation}
	Differentiating \eqref{eq:constant_eq_frac} with respect to $z$ gives
	\begin{equation}
		\label{eq:zero_eq_frac}
		0
		= \alpha(\alpha-1)(\alpha-2)\sum_{j=0}^m \lambda_j (z+j)^{\alpha-3}
	\end{equation}
	for all $z>2R$.
	
Since $\alpha\in\mathbb{Q}\setminus\mathbb{Z}_{\ge 0}$, we have
$\alpha,\alpha-1,\alpha-2\neq 0$, so $\alpha(\alpha-1)(\alpha-2)\neq 0$.
Let $\beta := \alpha-3 \in \mathbb{Q}\setminus\mathbb{Z}_{\ge 0}$ and define
$\lambda'_j := \alpha(\alpha-1)(\alpha-2)\lambda_j \in \mathbb{Q}$.
Then
\[
\sum_{j=0}^m \lambda'_j (z+j)^\beta = 0
\quad\text{for all } z>2R.
\]
By Lemma~\ref{lem:indep} it follows that  $\lambda_j=0$ for all $j$.
	But then \eqref{eq:frac_representation} reduces to $x y = 0$ for all
	$x,y>R$, a contradiction. Therefore no such finite representation
	\eqref{eq:frac_representation} can exist.
\end{proof}

\end{document}